\documentclass{article}

\usepackage{arxiv}
\usepackage[utf8]{inputenc} 
\usepackage[T1]{fontenc}    
\usepackage{hyperref}       
\usepackage{url}            
\usepackage{booktabs}       
\usepackage{amsfonts}       
\usepackage{nicefrac}       
\usepackage{microtype}      
\usepackage{lipsum}
\usepackage{graphicx}
\graphicspath{ {./images/} }

\usepackage{times}
\usepackage{comment}
\usepackage{multirow}
\usepackage{tcolorbox}
\tcbuselibrary{breakable} 

\definecolor{hellblue}{RGB}{220, 230, 241}

\tcbset{
    colback=hellblue,
    colframe=hellblue,
    arc=2mm,
    boxrule=0pt,
    width=\textwidth,      
    enlarge left by=0pt,   
    enlarge right by=0pt,  
    left=0pt,              
    right=0pt,             
    top=0pt,
    bottom=0pt,
    before skip=\baselineskip,
    after skip=\baselineskip,
    breakable,
    before upper={\setlength{\parindent}{15pt}} 
}

\newenvironment{mybox}{%
    \begin{tcolorbox}
}{%
    \end{tcolorbox}
}
\usepackage{mdframed}

\usepackage{amsthm}

\usepackage{comment}
\usepackage{amsmath,amssymb}
\usepackage{mathtools}
\usepackage{enumitem}
\usepackage{booktabs}
\usepackage{xcolor}
\usepackage{makecell}

\usepackage{xcolor}

\usepackage{float}
\DeclareMathOperator{\Vol}{Vol}
\providecommand{\abs}[1]{\lvert#1\rvert}
\providecommand{\norm}[1]{\lVert#1\rVert}
\providecommand{\R}{\mathbb{R}}

\newcommand{\Prob}{\mathbb{P}}
\newcommand{\E}{\mathbb{E}}
\newcommand{\Normal}{\mathcal{N}}
\newcommand{\Wtwo}{\mathcal{W}_2}
\newcommand{\inputs}{\mathcal{X}}
\DeclareMathOperator{\tr}{tr}
\DeclareMathOperator{\Var}{Var}
\DeclareMathOperator{\Cov}{Cov}
\DeclareMathOperator{\rank}{rank}

\newcommand{\Wk}{\mathcal{W}_{2,k^d}}
\newtheorem{assumption}{Assumption}
\providecommand{\range}{\operatorname{range}}

\DeclareMathOperator{\erf}{erf}
\DeclareMathOperator{\ReLU}{ReLU}
\DeclareMathOperator{\GELU}{GELU}
\DeclareMathOperator{\SiLU}{SiLU}

\usepackage{tikz}
\usetikzlibrary{arrows.meta,positioning}
\usepackage{soul}
\sethlcolor{yellow}
\AtBeginDocument{\soulregister{\ref}{7}\soulregister{\eqref}{7}}

\usepackage{amsthm}
\usepackage{natbib}

\newtheorem{definition}{Definition}

\newtheorem{theorem}{Theorem}
\newtheorem{lemma}{Lemma}
\newtheorem{corollary}{Corollary}
\newtheorem{remark}{Remark}
\newtheorem{proposition}{Proposition}

\author{
 Luis Wiederhold\thanks{Equal contribution} \\
  School of Computation, Information and Technology\\
Technical University of Munich\\
\texttt{wiederhold.luis@gmx.de}
   \And
 Nil Ayday$^*$  \\
 	School of Computation, Information and Technology\\
Technical University of Munich\\
  \texttt{nil.ayday@tum.de}
  \And
 Debarghya Ghoshdastidar \\
  School of Computation, Information and Technology\\
Technical University of Munich
}

\title{Lifting Finite-Sample Gaussian Approximation of Wide Neural Networks to Functional Convergence}
\begin{document}
\maketitle
\setcounter{footnote}{0}

\begin{abstract}%
Randomly initialised neural networks of width $m$ converge to Gaussian processes as $m \to \infty$, with two types of quantitative bounds characterising this convergence. Finite-sample bounds compare the network with its Gaussian limit at a fixed finite sample set, on an arbitrary domain, in the
quadratic Wasserstein distance $\mathcal W_2$, achieving a convergence rate of $m^{-1/2}$ with explicit constants. However, these constants
grow with the number of samples, so the bound becomes uninformative as the sample set is refined. Functional bounds remove that dependence by comparing the two as random
functions, at the price of spherical input, metrics weaker than $\mathcal W_2$, or slower rates. In this paper, we therefore lift a finite-sample bound to a functional one. The lifting
refines the sample set until it is dense in the domain, so the bound must first be made
independent of the number of samples. We achieve this by measuring the error per sample,
which makes the constant uniform. For a polynomial activation and any compact input set, this gives
a rate of $m^{-1/2}$ in the supremum norm, with an explicit constant. 


\end{abstract}

\keywords{neural network Gaussian process, functional convergence rates, Wasserstein distance}

\section{Introduction}


The connection between wide neural networks and Gaussian processes (GPs) was first established in \citet{Neal1996}, which showed that a randomly initialised single-hidden-layer network converges in distribution to a GP as its width grows. The limiting neural network Gaussian Process (NNGP) is fully specified by its covariance kernel, which depends explicitly on the network's activation function and initialisation. For many standard activations, this kernel can be computed in closed form \citep{DBLP:conf/nips/Williams96}.
This convergence reduces the network to its kernel, enabling exact Bayesian inference for infinite-width networks \citep{DBLP:conf/iclr/LeeBNSPS18}, to distill large datasets into a handful of synthetic examples by optimizing them directly through the kernel \citep{DBLP:conf/nips/NguyenNXL21}, and to rank neural architectures without training them \citep{DBLP:conf/iclr/ChenGW21}. Beyond single-layer neural networks, this equivalence extends to architectures
such as deep fully connected networks
\citep{DBLP:conf/iclr/LeeBNSPS18,DBLP:conf/iclr/MatthewsHRTG18}. Convergence to
GPs has also been established for wide feedforward and recurrent architectures \citep{DBLP:journals/corr/abs-1910-12478},
with analogous kernels derived for transformers \citep{DBLP:conf/icml/HronBSN20},
graph convolutional networks \citep{DBLP:conf/iclr/NiuA023}, and graph
transformers \citep{3844975.3844987}. The same framework has since led to more
refined and practical approximations, such as tensor programs and  the maximal update
parametrisation ($\mu$P)
\citep{DBLP:conf/icml/YangH21, DBLP:journals/corr/abs-2203-03466}.

Although the qualitative picture is now well understood, the convergence rate to NNGP in
terms of width $m$ remains less developed. Quantitative convergence has been studied from
two distinct perspectives: comparing the finite-dimensional distributions of the network
at a fixed finite sample set, or treating the network and its limit as random functions to
establish convergence on a function space. Table~\ref{tab:related} collects the results of
both perspective.
\begin{table}[tb]
\centering
\small
\setlength{\tabcolsep}{3.5pt}
\renewcommand{\arraystretch}{1.0}
\begin{tabular}{@{}cl@{}cccccc@{}}
\toprule
 & \textbf{Result} & \textbf{Depth} & \textbf{Inputs} & \textbf{Metric} & \textbf{Activation} & \textbf{Rate} & \textbf{with} $\boldsymbol{n}$ \\
\midrule
\multirow{6}{*}{\rotatebox[origin=c]{90}{finite sample}}
  & \citep[Thm.~3.1]{DBLP:journals/ml/BasteriT24}
   & any  & finite\hphantom{$^{*}$} & $\mathcal W_2$ & Lipschitz & $m^{-1/2}$
   & $\sqrt n$ {\scriptsize\color{gray}\S\ref{sec:input_dependence}} \\
 & \citep[Thm.~6]{DBLP:conf/aabi/BordinoFF24}
   & $1$  & finite\hphantom{$^{*}$} & $\mathcal W_1$ & $\mathcal C^{2}$, poly.\,gr. & $m^{-1/2}$
   & $\sqrt n$ \hphantom{{\scriptsize\S\ref{sec:input_dependence}}}  \\
 & \citep[Thm.1.1]{DBLP:journals/corr/abs-2507-12686}
   & any  & finite\hphantom{$^{*}$} & $\mathcal W_1$ & Lipschitz & $m^{-1/6}$
   & \multirow{4}{*}{\rotatebox[origin=c]{90}{implicit}} \\
 & \citep[Thm.~3.6(1)]{DBLP:journals/corr/abs-2307-06092}
   & any & finite$^{*}$ & $d_{\mathrm{cvx}}$ & poly.\,gr. & $m^{-1/2}$ & \\
 & \citep[Thm.~1.1(a)]{Trevisan2023WideDN}
   & any & finite$^{*}$ & $\mathcal W_p$ & Lipschitz & $m^{-1}$ & \\
 & \citep[Thm.~5]{DBLP:journals/corr/abs-2504-08335}
   & any & finite$^{*}$ & $\mathcal W_2$ & $\mathcal C^{1}$ or PL & $m^{-1}$ & \\
\midrule
\multirow{6}{*}{\rotatebox[origin=c]{90}{functional}}
 & \citep[Thm.~2]{DBLP:conf/colt/EldanMS21}
   & $1$ & sphere & $\mathcal W_1(\mathcal C^{0})$ & polynomial & $m^{-1/6}$
   & \multirow{6}{*}{\rotatebox[origin=c]{90}{none}} \\
 & \citep[Thm.~5]{DBLP:conf/colt/EldanMS21}
   & $1$ & sphere & $\mathcal W_1(\mathcal C^{0})$ & monomial & $m^{-1/2}$ & \\
 & \citep[Thm.~1]{DBLP:conf/colt/Klukowski22}
   & $1$ & sphere & $\mathcal W_2(L^{2})$ & polynomial & $m^{-1/2}$ & \\
 & \citep[Thm.~1.2]{DBLP:journals/corr/abs-2306-16308}
   & any & sphere & $\mathcal W_1(\mathcal C^{0})$ & Lipschitz & $m^{-q_d}$ & \\
 & \citep[Thm.~3.16]{DBLP:journals/corr/abs-2307-06092}
   & any & ball & $\mathcal W_2(\mathcal C^{k})$ & smooth & $m^{-1/8}$ & \\
\cmidrule(l){2-7}
 & This paper (Thm.~\ref{thm:functional})
   & $1$ & compact & $\mathcal W_2(\mathcal C^{0})$ & polynomial & $m^{-1/2}$ & \\
\bottomrule
\end{tabular}
\caption{Quantitative Gaussian approximation for wide networks; rates are in the width
$m$, and $n$ is the number of inputs. Finite-sample bounds hold at a fixed sample set with
a constant growing in $n$; ``implicit'' means the bound depends on the input set, but this dependence is not quantified. Functional bounds avoid this at the cost of the input set, metric or rate.
finite$^{*}$: finite input, non-degenerate kernel; $d_{\mathrm{cvx}}$: convex distance; poly.\,gr.: polynomial growth; PL: piecewise linear.
Rates for \citet{DBLP:journals/corr/abs-2507-12686} and
\citet{DBLP:journals/corr/abs-2306-16308} are one-hidden-layer specialisations. For the latter, inputs lie on the unit sphere of $\R^{d}$ and $q_d = \frac{1}{8d-2}-\varepsilon$
for arbitrary $\varepsilon>0$.
By Corollary~\ref{cor:compact} the last row holds on every compact subset of $\R^{d}$,
with the constant of an enclosing hyperrectangle.}
\label{tab:related}
\end{table}

\paragraph{Finite-Sample Convergence.} The first approach fixes a finite sample set $\mathcal{X}=\{x_1,\dots,x_n\}$ and compares the random vector $f[\mathcal{X}]$ of $n$ network evaluations with the corresponding Gaussian vector $G[\mathcal{X}]$. The strength of the finite-input route is its broad applicability. Once the comparison is reduced to a pair of random vectors in $\mathbb{R}^n$, the tools of finite-dimensional Gaussian approximation apply: the distance between two centred Gaussian laws is controlled by their covariances, and couplings can be constructed explicitly \citep{DBLP:journals/ml/BasteriT24}. This yields bounds in the quadratic Wasserstein distance $\mathcal W_2$, a stronger metric
than the first-order $\mathcal W_1$, with explicit constants, and it asks nothing of the
input points themselves: they may form an arbitrary finite collection in any dimension, rather than requiring restrictive structural assumptions such as lying on a unit sphere. It also covers the standard Gaussian initialisation with bias terms and arbitrary depth \citep{DBLP:journals/ml/BasteriT24}. Working with a finite sample set, however, introduces a dependence on its size.

A bound at \(\mathcal{X}\) is a statement about a random vector in \(\mathbb{R}^n\), and with \(n\) fixed, quantities that vary with \(n\) are absorbed into the constants.
For example in \cite{DBLP:journals/ml/BasteriT24} these include the Frobenius norm of a \(n\times n\) covariance fluctuation and the smallest strictly positive eigenvalue of the limiting covariance matrix $K[\mathcal{X}]$.  Whether the bound survives refinement of the sample set depends on how these scale. In a learning-theoretic application the sample set is the
training sample, so $n$ needs to be large. \citet{DBLP:conf/icml/AroraDHLW19} bound the population
error of kernel regression with the limiting kernel by a term of order $n^{-1/2}$, and
carrying that guarantee over to a network of width $m$ means comparing the two on that
same sample. An $n$-dependent constant that grows with the sample size can offset the
statistical decay. Moreover, a finite-input bound describes the network only at the chosen inputs. Convergence of the finite-dimensional distributions is a statement about
each fixed finite sample set, so it does not imply convergence of the laws on a function
space. The second perspective treats the network directly as a random function.

\paragraph{Functional Convergence.}
\citet{DBLP:conf/iclr/BracaleFFP21} and
\citet{DBLP:journals/corr/abs-2107-01562} establish functional convergence without rates.
The lower half of Table~\ref{tab:related} collects the quantative results. These avoid the dependence on a
prescribed finite sample set, but each restricts the setting elsewhere, in the domain or in the metric. The inputs lie on a sphere in
\cite{DBLP:conf/colt/EldanMS21,DBLP:conf/colt/Klukowski22,DBLP:journals/corr/abs-2306-16308},
and \citet{DBLP:conf/colt/Klukowski22} draws the weights from a sphere as well.  A sphere
has no interior, so such a bound says nothing about a ball, a hyperrectangle, or any set with positive volume. The metric is the weaker
$\mathcal W_1$ distance in
\cite{DBLP:conf/colt/EldanMS21,DBLP:journals/corr/abs-2306-16308}. Since $\mathcal W_1\le\mathcal W_2$, a bound in $\mathcal W_2$ implies one in
$\mathcal W_1$ but not conversely.
Closest to our
setting is \citet{DBLP:journals/corr/abs-2307-06092} with an open ball in any dimension, arbitrary depth, the same Gaussian initialisation
with biases, and a $\mathcal C^{k}$ norm stronger than ours, since it controls derivatives
as well. Unlike a sphere, their ball is full-dimensional but they pay for the larger domain and the stronger metric with the rate $m^{-1/8}$, against the $m^{-1/2}$ available on
the sphere. Apart from the last row, no functional row of Table~\ref{tab:related} combines a general
compact set of inputs with the rate $m^{-1/2}$.


\begin{mybox}
The two perspectives are therefore complementary: finite-input bounds are general in
the domain, the initialisation and the metric but carry a constant that depends on
the input set. In contrast, functional bounds remove that dependence at the price of
restricted domains and metrics.
To our knowledge, no functional bound in the quadratic Wasserstein distance with
respect to the supremum norm attains the rate $m^{-1/2}$ on a
general compact domain under the standard Gaussian initialisation.
A functional
bound with the generality of the finite-input analysis would give guarantees that
do not weaken as the sample grows.
\end{mybox}

\paragraph{Contribution.} Our result combines the two perspectives. We start from the finite-sample bound of
\cite{DBLP:journals/ml/BasteriT24} and prove that under an effective-rank condition, which we verify for polynomial activations, the constant of \cite{DBLP:journals/ml/BasteriT24} is at most of order $\sqrt n$. The
bound therefore reads $\sqrt{n/m}$, so the width must grow at least as fast as the number
of inputs for it to be informative. Measuring the error per input removes this growth entirely. Then we provide a chain of results that converts this finite-sample bound into
a functional bound in the supremum norm. For one hidden layer, any polynomial activation, and inputs in a hyperrectangle, we obtain $\mathcal{W}_2^{C^{0}(I)}(f,G) = O(m^{-1/2})$ with explicit constants, and by restriction
(Corollary~\ref{cor:compact}) the bound holds on every compact subset of $\R^{d}$. Among the functional bounds of Table~\ref{tab:related}, the rate $m^{-1/2}$ appears only
on the sphere. Theorem~\ref{thm:functional} reaches it on an arbitrary compact set, so the
spherical setting of
\cite{DBLP:conf/colt/EldanMS21,DBLP:conf/colt/Klukowski22,DBLP:journals/corr/abs-2306-16308}
is a special case. This is a first step towards merging
the two formalisations of quantitative convergence, carrying the generality of the
finite-sample bounds into a statement about the random function.

Section~\ref{sec:prelim} fixes notation, introduces the network and its Gaussian process
limit, and states the finite-sample bound we start from, whose constant grows at most like
$\sqrt n$ and becomes independent of the input set once the error is measured per input.
Section~\ref{sec:functional} turns that uniformity into a functional bound by refining the
input set until it is dense in the domain, a limit that is available only because the
constant no longer grows with the number of inputs.
Section~\ref{sec:discussion} discusses the open questions.

\section{Preliminaries and notes on finite-sample bound}\label{sec:prelim}
We fix here the objects and notation used throughout: the network and its Gaussian process
limit (Section~\ref{sec:pr_nngp}), the distance in which the two are compared (Section~\ref{sec:input_dependence}), and the finite-sample bound we start
from, whose constant we track as the sample set grows (Section~\ref{sec:input_dependence}).
\paragraph{Notation.}
For $v\in\R^{n}$ we write $\norm{v}=\bigl(\sum_{i=1}^{n}v_i^2\bigr)^{1/2}$
for the Euclidean norm, and for a matrix $A=(A_{ij})$ we write
$\norm{A}=\bigl(\sum_{i,j}A_{ij}^2\bigr)^{1/2}$
for the Frobenius norm, which coincides with the Euclidean norm when $A$ is a vector.
For a symmetric positive semi-definite $A\in\R^{n\times n}$ with eigenvalues $\lambda_1,\dots,\lambda_n\ge 0$ we write
\( \lambda_{\min}(A)=\min_i\lambda_i\),
 \( \lambda^{+}(A)=\min\{\lambda_i:\lambda_i>0\},
\) and $\sqrt{A}$ for the symmetric positive semi-definite square root, obtained by
replacing each $\lambda_i$ by $\sqrt{\lambda_i}$ in the spectral decomposition.
$\lambda_{\min}$ and $\lambda^{+}$ differ exactly when $A$ is singular, which is the regime of
Section~\ref{sec:input_dependence}. For a function $h$ and a finite set $\mathcal T=\{t_1,\dots,t_n\}$ in its domain, fixed
together with an enumeration, we write $h[\mathcal T]=(h(t_1),\dots,h(t_n))$ for the vector
of evaluations; if $h$ is $\R^{m}$-valued, $h[\mathcal T]$ is the $m\times n$ matrix whose
$i$-th column is $h(t_i)$. For a multi-index $\alpha=(\alpha_1,\dots,\alpha_d)\in\mathbb{N}_0^{d}$ we write
$\abs{\alpha}=\sum_{j=1}^{d}\alpha_j$ and
$x^{\alpha}=x_1^{\alpha_1}\cdots x_d^{\alpha_d}$ for $x\in\R^{d}$. The
polynomials of total degree at most $r$ in $d$ variables form a vector space of
dimension $\binom{r+d}{d}$, spanned by $\{x^{\alpha}:\abs{\alpha}\le r\}$.

\paragraph{Gaussian vectors.} For a random vector $X\in\mathbb{R}^n$ with finite second moment, its covariance matrix is defined as
\(
\Sigma(X)=\mathbb{E}\bigl[(X-\mathbb{E}[X])(X-\mathbb{E}[X])^\top\bigr].
\) If
$X$ is centred then $\Sigma(X)=\E[XX^{\top}]$. 
We write $\mathcal N(0,\Sigma)$ for the centred Gaussian law on $\R^{n}$ with covariance
$\Sigma$.

\subsection{Random neural networks and their Gaussian process (GP) limit}
\label{sec:pr_nngp}
We now define the two objects the paper compares. The first is a network at
finite width $m$; the second is the GP it approaches as the width
grows ($m \to \infty$), whose kernel is determined by the activation and the initialisation
variances.

\begin{definition}[randomly initialised shallow network]\label{def:network}
Let $\phi:\R\to\R$ be an activation function. A \emph{shallow neural network} of width $m$ is a function
$f:\R^{d}\to\R$ of the form
\[
  g(x)=Wx+b\in\R^{m},
  \qquad
  f(x)=a^{\top}\phi\bigl(g(x)\bigr)+c\in\R ,
\]
with parameters $W\in\R^{m\times d}$, $b,a\in\R^{m}$, $c\in\R$, where $\phi$ acts
componentwise and $g$ is the \emph{pre-activation}. At \emph{random initialisation} the
parameters are independent, with
\(
  W_{rj}\sim\Normal\Bigl(0,\tfrac{\sigma_w^2}{d}\Bigr),\quad
  b_r\sim\Normal\bigl(0,\sigma_b^2\bigr),\quad
  a_r\sim\Normal\Bigl(0,\tfrac{\tau_a^2}{m}\Bigr),\quad
  c\sim\Normal\bigl(0,\tau_c^2\bigr),
\)
with $\sigma_w^2,\tau_a^2>0$ and $\sigma_b^2,\tau_c^2\ge 0$.
\end{definition}

The network in Definition~\ref{def:network} maps $d$-dimensional inputs to scalar outputs. Multivariate outputs can be handled analogously and are omitted for notational simplicity. Throughout, $I\subset\R^{d}$ denotes a fixed compact set, and input sets
$\inputs=\{x_1,\dots,x_n\}$ are finite subsets of $I$, with corresponding
evaluations $f[\inputs]\in\R^{n}$ and $g[\inputs]\in\R^{m\times n}$. Since $W$
and $b$ are Gaussian, the pre-activation is a GP for every width.
Applying $\phi$ destroys Gaussianity, but the output is a sum of $m$ independent
terms, so by the central limit theorem it is asymptotically Gaussian as
$m\to\infty$. The covariance of that limit is computable, and the two kernels are
the following.

\begin{definition}[Gaussian process and NNGP kernels]\label{def:nngp}
A \emph{Gaussian process} on $I$ is a family $(G(x))_{x\in I}$ of real random variables on
a common probability space, every finite subfamily of which is a Gaussian vector; it is
\emph{centred} if $\E[G(x)]=0$ for every $x\in I$, and its \emph{covariance kernel} is
$(x,y)\mapsto\E[G(x)G(y)]$. A centred Gaussian process is determined in law by its
covariance kernel, and every symmetric positive semi-definite function on $I\times I$ is
the covariance kernel of one \citep[Theorem~6.16]{kallenberg2002foundations}.
For $x,y\in I$ define
\begin{align}
  \kappa(x,y)&=\tfrac{\sigma_w^2}{d}\,\langle x,y\rangle+\sigma_b^2,
    \label{eq:kappa}\\[2pt]
  K(x,y)&=\tau_a^2\,\E\bigl[\phi(Z_x)\,\phi(Z_y)\bigr],
    \qquad (Z_x,Z_y)\sim\Normal\Bigl(0,
      \begin{psmallmatrix}\kappa(x,x)&\kappa(x,y)\\
        \kappa(x,y)&\kappa(y,y)\end{psmallmatrix}\Bigr),
    \label{eq:K}\\[2pt]
  K_c(x,y)&=K(x,y)+\tau_c^2 .
    \label{eq:Kb}
\end{align}
Here $\kappa$ is the covariance of the pre-activation coordinates of
Definition~\ref{def:network}, the same for every coordinate, and for a finite
$\inputs=\{x_1,\dots,x_n\}\subseteq I$ we write $\kappa[\inputs]$, $K[\inputs]$,
$K_c[\inputs]\in\R^{n\times n}$ for the matrices with entries $\kappa(x_i,x_j)$,
$K(x_i,x_j)$, $K_c(x_i,x_j)$. All three kernels are symmetric and positive semi-definite,
so the \emph{neural network Gaussian process} (NNGP) associated with
Definition~\ref{def:network} is well defined as the centred Gaussian process $G$ on $I$
with covariance kernel $K_c$.
\end{definition}

The pre-activation $g[\inputs]$ consists of $m$ independent rows with law
$\Normal(0,\kappa[\inputs])$ and is exactly Gaussian at every width. The output
$f[\inputs]$ is not. A finite-input Gaussian approximation bound such as
Theorem~\ref{thm:bt} measures, at finite width, how far the law of $f[\inputs]$ is from
$\Normal(0,K_c[\inputs])$. It is stated in
the quadratic Wasserstein distance, defined next.


\subsection{The quadratic Wasserstein distance}
\label{sec:pr_distance}
\begin{definition}[quadratic Wasserstein distance]\label{def:w2}
Let $(E,\norm{\cdot}_E)$ be a separable Banach space and let $p,q$ be probability
measures on $E$ with finite second moment. Then
\[
  \Wtwo^{E}(p,q)
  =\inf\Bigl\{\sqrt{\E\bigl[\norm{X-Y}_E^{2}\bigr]}
    \;:\; X\sim p,\ Y\sim q\Bigr\},
\]
the infimum running over all \emph{couplings} of $p$ and $q$, i.e.\ all pairs
$(X,Y)$ of $E$-valued random variables on a common probability space with the
given marginals. We write $\Wtwo^{E}(X,Y)$ for $\Wtwo^{E}(\Prob_X,\Prob_Y)$. Four instances occur below, differing only in the state space $E$ or its norm:
\begin{center}
\small
\renewcommand{\arraystretch}{1.5}
\setlength{\tabcolsep}{5pt}
\begin{tabular}{@{}lllll@{}}
\toprule
$E$ & $\norm{\cdot}_E$ & notation & compares & used in \\
\midrule
$\R^{n}$
  & $\norm{v}=\bigl(\sum_{i=1}^{n}v_i^{2}\bigr)^{1/2}$
  & $\Wtwo=\Wtwo^{\R^{n}}$
  & finite sample set, total error
  & Thm.~\ref{thm:bt} \\
$\R^{n}$
  & $\norm{v}_{n}=\bigl(n^{-1}\sum_{i=1}^{n}v_i^{2}\bigr)^{1/2}$
  & $\mathcal{W}_{2,n}$
  & finite sample set, error per sample
  & \S\ref{sec:input_dependence} \\
$L^{2}(I)$
  & $\bigl(\int_I h(x)^{2}\,dx\bigr)^{1/2}$
  & $\Wtwo^{L^{2}(I)}$
  & random functions, mean-square error
  & \S\ref{sec:functional} \\
$\mathcal C^{0}(I)$
  & $\sup_{x\in I}\abs{h(x)}$
  & $\Wtwo^{\mathcal C^{0}(I)}$
  & random functions, supremum error
  & \S\ref{sec:functional} \\
\bottomrule
\end{tabular}
\end{center}
All four are separable Banach spaces. In $\R^{n}$ the norm $\norm{\cdot}_{n}$ is the root mean square
of the coordinates, so
$\mathcal{W}_{2,n}=n^{-1/2}\Wtwo$ measures the transport cost \emph{per input}, where $\Wtwo=\Wtwo^{\R^{n}}$ is the Euclidean case; it is again
a metric, being a positive multiple of one. Since
$\norm{h}_{L^{2}(I)}\le\sqrt{\Vol(I)}\,\norm{h}_{\mathcal C^{0}(I)}$, we have
$\Wtwo^{L^{2}(I)}\le\sqrt{\Vol(I)}\,\Wtwo^{\mathcal C^{0}(I)}$, so the
supremum-norm statement is the strongest of the two functional bounds.
\end{definition}

\begin{remark}[coupling convention]\label{rem:coupling}
Since $\Wtwo^{E}$ is an infimum over couplings, a single explicit coupling bounds $\Wtwo^{E}$ from above (Proposition~\ref{prop:w2}\,(1)).
Throughout, $(f,G)$ therefore denotes a pair defined on a common probability space with
the laws of Definitions~\ref{def:network} and~\ref{def:nngp}: all bounds on
$\norm{f-G}_{E}$ are proved for almost every realisation, and the expectation and the infimum are taken only
at the end. Further properties of $\Wtwo$ are collected in Appendix~\ref{app:w2}.
\end{remark}

\subsection{The finite-sample bound and its data dependence}
\label{sec:input_dependence}

Quadratic Wasserstein distance is well suited for Gaussian approximation since the distance between two
centred Gaussian laws is bounded by the Frobenius distance of the square roots of
their covariances, so covariance perturbations can be tracked explicitly. This is the mechanism behind the bound of \citet{DBLP:journals/ml/BasteriT24}, revisited in
Theorem~\ref{thm:bt}. We state their result under the following standard assumptions about the activation function.
\begin{assumption}[activation function]\label{ass:activation}
The activation $\phi:\R\to\R$ of Definition~\ref{def:network} satisfies:
\begin{enumerate}[label=\textup{(\arabic*)},ref=\arabic*,leftmargin=2.2em,
                  topsep=2pt,itemsep=2pt]
  \item\label{it:growth} \emph{Polynomial growth.} There are $A,p<\infty$ with
    $\abs{\phi(z)}\le A(1+\abs{z})^{p}$ for all $z\in\R$.
  \item\label{it:lipschitz} \emph{Locally Lipschitz with derivative of
    polynomial growth.} $\phi$ is absolutely continuous and there are
    $A',p'<\infty$ with $\abs{\phi'(z)}\le A'(1+\abs{z})^{p'}$ for almost every\
    $z\in\R$.
  \item\label{it:nondeg} \emph{Non-degeneracy.} For every $x\in I$ the variance
    $\kappa(x,x)$ is strictly positive and $\phi$ does not vanish
    $\Normal\bigl(0,\kappa(x,x)\bigr)$-almost everywhere.
\end{enumerate}
\end{assumption}

Assumption~\ref{ass:activation} is satisfied by $\tanh$, the logistic sigmoid,
$\erf$, $\ReLU$ and its leaky and smooth variants ($\GELU$, $\SiLU$, softplus),
by every polynomial, and by products such as $z\tanh(z)$. It excludes exponential
growth, which violates Assumption~\ref{ass:activation}\eqref{it:growth}; unbounded derivatives at the origin such
as $\abs{z}^{\alpha}$ with $\alpha\in(0,1)$, which violate Assumption~\ref{ass:activation}\eqref{it:lipschitz};
and jump discontinuities, which are not absolutely continuous. Assumption~\ref{ass:activation}\eqref{it:nondeg} also involves the bias and the
domain: since $\kappa(x,x) = \frac{\sigma_w^2}{d}\|x\|^2 + \sigma_b^2$, it holds for
all the above activations if and only if $\sigma_b>0$ or $0\notin I$.

\begin{theorem}[{adapted from \cite[Theorem~3.1]{DBLP:journals/ml/BasteriT24}}]\label{thm:bt}
Let $\inputs = \{x_1,\dots,x_n\} \subset I$ be a finite sample set, let $f$ be the
network of Definition~\ref{def:network} and let $G$ be the Gaussian process of
Definition~\ref{def:nngp}. Under Assumption~\ref{ass:activation}, for every width
$m \in \mathbb N$,
\begin{equation}\label{eq:bt}
  \Wtwo\bigl(f[\inputs],\, G[\inputs]\bigr)
  \;\le\; \frac{C(\phi,\inputs)}{\sqrt m},
\end{equation}
with the explicit constant
\begin{equation}\label{eq:bt-constant}
  C(\phi,\inputs)
  \;=\;
  \frac{\sqrt{\E\Bigl[\bigl\lVert
      \tau_a^2\,\phi(Z)\phi(Z)^{\top} - K[\inputs]
    \bigr\rVert^{2}\Bigr]}}
    {\sqrt{\lambda^{+}\bigl(K[\inputs]\bigr)}},
  \qquad
  Z \sim \mathcal N\bigl(0,\kappa[\inputs]\bigr) \in \R^{n}.
\end{equation}
The output bias contributes the same independent translation to $f[\inputs]$ and to
$G[\inputs]$, so by Proposition~\ref{prop:w2}\,(3) it cancels and the constant involves
$K[\inputs]$ rather than $K_c[\inputs]$. $C(\phi,\inputs)$ depends only on the activation $\phi$, the input
set $\inputs$ and the variances $\sigma_w^2,\sigma_b^2,\tau_a^2$, but not on the
width $m$.
\end{theorem}
Theorem~3.1 of \citet{DBLP:journals/ml/BasteriT24} makes the stronger assumption that $\phi$ is globally Lipschitz. The proof under Assumption~\ref{ass:activation} is given in Appendix~\ref{sec:proofs_prelim}. The bound of Theorem~\ref{thm:bt} holds at any fixed finite sample set, but its constant
$C(\phi,\inputs)$ depends on that set.  Section~\ref{sec:functional} evaluates $f$ and $G$ on grids that become dense in $I$, and
hence arbitrarily large, so the growth of $C(\phi,\inputs)$ decides whether a functional
comparison is possible. This remainder of this section determines how $C(\phi,\inputs)$ behaves: the numerator of \eqref{eq:bt-constant} grows at most linearly in the number of
inputs, while the denominator admits no lower bound in general, so that enlarging the
input set can render \eqref{eq:bt} vacuous. A lower bound on the denominator must therefore be assumed. Linear growth in $\sqrt{n}$ is the most that can be
demanded, since $\lambda^{+}(K[\inputs])$ never exceeds $\tr K[\inputs]$, which grows
linearly under Assumption~\ref{ass:activation}.

\begin{assumption}[effective-rank condition]\label{ass:eff-rank}
A collection $\mathcal F$ of finite sample sets $\inputs\subset I$ satisfies the
\emph{effective-rank condition with constant $\nu>0$} if, writing $n=\abs{\inputs}$,
\begin{equation}\label{eq:eff-rank}
  \lambda^{+}\bigl(K[\inputs]\bigr)\;\ge\;\nu\,n
  \qquad\text{for every }\inputs\in\mathcal F .
\end{equation}
\end{assumption}

Assumption~\ref{ass:eff-rank} excludes every
activation for which the $\operatorname{rank} (K[\mathcal X])$ is unbounded (see proof of Proposition~\ref{cor:constant} for details). In particular every activation whose
kernel $K$ is strictly positive definite on $I$ is excluded, since then
$\operatorname{rank} (K[\mathcal X]) = |\mathcal X|$. 
\footnote{The rows of Table~\ref{tab:related} that reach $m^{-1}$ do so under exactly this
non-degeneracy, so the $m^{-1/2}$ rate of Theorem~\ref{thm:bt} is the best available rate in our setting.} The condition is therefore a
finite-rank condition, and Lemma~\ref{lem:polynomial-er} verifies it for polynomial activations on the grids of Section~\ref{sec:functional}. Under Assumption~\ref{ass:activation} the numerator of the constant in Equation~\eqref{eq:bt-constant} grows at most linearly in $n$, while under Assumption~\ref{ass:eff-rank} the
denominator grows like $\sqrt n$, leaving the bound of
Theorem~\ref{thm:bt} with a factor $\sqrt n$ and a constant that depends neither on the
sample set nor on the width. The following proposition makes these statements precise, and 
is proved in Appendix~\ref{sec:proofs_prelim}.

\begin{proposition}[the constant of Theorem~\ref{thm:bt}]\label{cor:constant}
Let $\inputs\subset I$ be a finite sample set with $n=\abs{\inputs}$.
\begin{enumerate}
\item Under Assumption~\ref{ass:activation} with growth constants $A,p$, the numerator of
  \eqref{eq:bt-constant},
\[
  N:=\sqrt{\E\Bigl[\bigl\lVert\tau_a^2\,\phi(Z)\phi(Z)^{\top}
      -K[\inputs]\bigr\rVert^{2}\Bigr]},
  \qquad Z\sim\Normal\bigl(0,\kappa[\inputs]\bigr)\in\R^{n},
\]
  satisfies $N\le \tau_a^2 A^2\sqrt{M_{4p}}\,n$, where
\[
  M_{s}:=\E\bigl[(1+\abs{\xi})^{s}\bigr]<\infty
  \quad\text{for }s>0 \text{ and } \xi\sim\Normal(0,\kappa_I^{2}) \text{ with } \kappa_I^{2}:=\max_{x\in I}\kappa(x,x)<\infty.
\]
\item If in addition $K[\inputs]$ satisfies Equation \eqref{eq:eff-rank} with constant $\nu>0$, then
  for every width $m\in\mathbb N$,
  \[
    \Wtwo\bigl(f[\inputs],G[\inputs]\bigr)\;\le\;D\sqrt{\frac{n}{m}},
  \]
  with $D:=\tau_a^{2}A^{2}\sqrt{M_{4p}/\nu}$.
\end{enumerate}
The constants depend only on $\phi$ through $A,p$, on $I$, on the variances
$\sigma_w^2,\sigma_b^2,\tau_a^2$ and, in part (2), on $\nu$, and neither on $\inputs$ nor
on $m$.
\end{proposition}

Proposition~\ref{cor:constant} shows that the bound of Theorem~\ref{thm:bt} grows with the
number of inputs: the width $m$ must grow faster than $n$ for it to be informative. This
is the main limitation of finite-sample bounds. Section~\ref{sec:functional} removes that
growth and compares $f$ and $G$ as random functions.
\section{From a Finite-Input Bound to Functional Convergence}
\label{sec:functional}

\begin{figure}[t]
\centering
\resizebox{\textwidth}{!}{%
\begin{tikzpicture}[
  >={Stealth[length=2.4mm]},
  rmterm/.style={inner sep=2pt},
  rmstep/.style={font=\footnotesize, align=center, inner sep=1.5pt},
]
\node[rmterm]                    (c0)   {$\norm{f-G}_{\mathcal C^{0}(I)}$};
\node[rmterm, right=26mm of c0]  (l2)   {$\norm{f-G}_{L^{2}(I)}$};
\node[rmterm, right=26mm of l2]  (gr)   {$\norm{(f-G)[\inputs_k]}_{n}$};
\node[rmterm, right=30mm of gr]  (rate) {$C^{\ast}/\sqrt m$};

\path (c0) -- node {$\lesssim$} node[rmstep, above=2mm]
  {Lemma ~\ref{lem:norm-equiv} } (l2);

\path (l2) -- node {$\lesssim$} node[rmstep, above=2mm]
  {Discretization\\[-2pt]{\scriptsize Lemma~\ref{lem:disc}}} (gr);

\path (gr) -- node {$\lesssim$} node[rmstep, above=2mm]
  {Coupling $+$ Uniform Wasserstein Distance\\[-2pt]%
   {\scriptsize Lemma~\ref{lem:coupling}, Equation~\ref{eq:uniform}}} (rate);
\end{tikzpicture}}
\caption{Roadmap of Section~\ref{sec:functional}. Each inequality bounds the quantity on its
left by the one on its right, and the reference names the result that does it. A norm equivalence bounds the supremum norm by the $L^{2}(I)$ norm, a discretization
bounds that by the grid values, and a coupling bounds those by the width. Composing the three steps proves
Theorem~\ref{thm:functional}.}
\label{fig:roadmap}
\end{figure}

Measured per input rather than in total, the bound of Proposition~\ref{cor:constant}
gives,
\begin{equation}\label{eq:uniform}
  \mathcal W_{2,n}\bigl(f[\inputs],G[\inputs]\bigr)\;\le\;\frac{D}{\sqrt m},
  \qquad D:=\tau_a^{2}A^{2}\sqrt{M_{4p}/\nu} .
\end{equation}
The right-hand side of \ref{eq:uniform} does not depend on the sample set. In the unrescaled form, the bound $D\sqrt{n/m}$ is informative only when
$n\ll m$, so the width must grow with the number of inputs for this bound to be useful, as discussed in the introduction. In the rescaled
distance, the width alone governs the
per-sample approximation error. But rescaling presupposes $n$, and \eqref{eq:uniform} is still a
statement about a random vector in $\R^{n}$. This section uses \eqref{eq:uniform} to get a bound on $f$ and $G$ as random functions.

Converting this bound into a statement about the function is a chain of three
steps, shown in Figure~\ref{fig:roadmap}, where each
arrow replaces the quantity on its left by a larger one. (i) A \emph{norm equivalence} bounds the supremum norm by the $L^{2}(I)$ norm on the
finite-dimensional space when $\phi$ is a
polynomial. (ii) A
\emph{discretization} estimate bounds the $L^{2}(I)$ norm by the values on a finite grid,
up to a mesh error that shrinks as the grid is refined. Refining the grid adds inputs, so with the unrescaled constant the mesh error and the
constant cannot be made small at once. This is the obstruction identified in
\citet[Section~5]{DBLP:journals/ml/BasteriT24}, which is why the functional bound is only
sketched there. Since the right-hand side of \eqref{eq:uniform} does not grow with the
number of inputs, the grid may be refined freely and the mesh error sent to zero.
(iii) An explicit \emph{coupling} then
identifies the remaining grid term with the left-hand side of \eqref{eq:uniform}. The three bounds are deterministic and are applied pathwise to $h=f-G$, for the coupled
pair of Remark~\ref{rem:coupling}, with the expectation taken only afterwards.

We take $I$ to be a hyperrectangle, so that the grids are products of $d$ one-dimensional
grids. The product structure gives the covering property that the discretization step
needs (Remark \ref{lem:grid} in Apendix~\ref{app:w2}). For a polynomial activation these grids also satisfy
Assumption~\ref{ass:eff-rank} with one constant $\nu$ for all $k\ge k_1$
(Lemma~\ref{lem:polynomial-er}). The hyperrectangle restriction is removed afterwards by
Corollary~\ref{cor:compact}.

\begin{definition}[midpoint grid]\label{def:grid}
Assume $I = I_1\times\cdots\times I_d$ be a $d$-dimensional hyperrectangle, with
$I_j=[a_j,b_j]$, $\abs{I_j}:=b_j-a_j>0$ and $\Vol(I):=\prod_{j=1}^{d}\abs{I_j}$.
For $k\in\mathbb N$ let
\[
  \inputs_k^{(j)}=\Bigl\{\,a_j+\bigl(l-\tfrac12\bigr)\tfrac{\abs{I_j}}{k}
  \;:\; l=1,\dots,k\,\Bigr\}\subset I_j
\]
be the one-dimensional midpoint grid of order $k$ on $I_j$, the superscript
indexing the coordinate. The midpoint grid of order $k$ on $I$ is
\begin{equation}\label{eq:grid}
  \inputs_k=\inputs_k^{(1)}\times\cdots\times\inputs_k^{(d)}\subset I .
\end{equation}
In particular $n=\abs{\inputs_k}=k^{d}$: the subscript $k$ counts points per axis, not
the total number of points. Further properties of $\inputs_k$ are collected in
Appendix~\ref{app:w2}.
\end{definition}

The grids of Definition~\ref{def:grid} form the collection to which
Assumption~\ref{ass:eff-rank} is applied. For a polynomial activation the following lemma
establishes it with a single constant $\nu$, so \eqref{eq:uniform} holds with the same $D$
for every $k$, which is what the limit taken in Theorem~\ref{thm:functional} requires.

\begin{lemma}[polynomial activations]\label{lem:polynomial-er}
Let $\phi$ be a nonzero polynomial of degree $r$ satisfying
Assumption~\ref{ass:activation}, and let $\mathcal X_k$ be the midpoint grids of
Definition~\ref{def:grid}. Then $\rank (K[\inputs_k])\le\binom{r+d}{d}$ for every $k$, and there are $\nu>0$ and
$k_1\in\mathbb N$ such that $\inputs_k$ satisfies the effective-rank condition
(Assumption~\ref{ass:eff-rank}) with constant $\nu$ for every $k\ge k_1$.
\end{lemma}

The proof of Lemma~\ref{lem:polynomial-er} factorises the kernel matrix as
$K[\inputs_k]=V_kMV_k^{\top}$, where the rows of $V_k\in\R^{k^{d}\times R}$ are the
monomials of total degree at most $r$ evaluated at the grid points, $R=\binom{r+d}{d}$, and
$M\in\R^{R\times R}$ is positive semi-definite and independent of $k$. Hence $\rank(K[\inputs_k])\le\rank(M)\le R$, however fine the grid is. The nonzero eigenvalues of $K[\inputs_k]$ coincide with those of the $R\times R$ matrix $k^{d}M^{1/2}(\frac{1}{k^{d}}V_k^{\top}V_k)M^{1/2}$, whose middle factor $\frac{1}{k^{d}}V_k^{\top}V_k$ is a midpoint Riemann sum. Its limit is positive definite, which yields $\lambda^{+}(K[\inputs_k])\ge\nu\,k^{d}$ for all $k\ge k_1$.

Our main result, Theorem~\ref{thm:functional}, provides a functional Wasserstein bound in the supremum norm (\(\Wtwo^{\mathcal C^{0}(I)}(f,G)\)). This bound compares the neural
network and its NNGP as random continuous functions on a compact interval for polynomial activation functions.

\begin{theorem}[main result: functional bound with $m^{-1/2}$ rate]\label{thm:functional}
Let $\phi$ be a nonzero polynomial of degree $r$ satisfying
Assumption~\ref{ass:activation}, let $I=I_1\times\cdots\times I_d$ be a hyperrectangle.
Then for every $m\in\mathbb N$,
\begin{equation}\label{eq:main}
  \Wtwo^{\mathcal C^{0}(I)}(f,G)\;\le\;\frac{C^{\ast}}{\sqrt m},
  \qquad
  C^{\ast}:=(r+1)^{d}\,\tau_a^{2}A^{2}\sqrt{\frac{2\,M_{4p}}{\nu}},
\end{equation}
where $A,p$ are the growth constants of Assumption~\ref{ass:activation},
$M_{4p}$ is the moment of Theorem~\ref{thm:numerator}, and $\nu>0$ is the effective-rank
constant supplied by Lemma~\ref{lem:polynomial-er}.
\end{theorem}

The supremum norm only decreases when the domain is made smaller, so
Theorem~\ref{thm:functional} carries over from the hyperrectangle to every compact set it
contains, with the same constant (Corollary~\ref{cor:compact}).

\begin{corollary}\label{cor:compact}
Let $S\subset\R^{d}$ be compact and let $I$ be a hyperrectangle containing $S$.
Under the hypotheses of Theorem~\ref{thm:functional} on $I$,
\[
  \Wtwo^{\mathcal C^{0}(S)}(f,G)\;\le\;\Wtwo^{\mathcal C^{0}(I)}(f,G)
  \;=\;O\bigl(m^{-1/2}\bigr).
\]
Since every compact set is contained in a hyperrectangle, the rate $m^{-1/2}$
holds on any compact set of inputs; only the constant depends on the enclosing
hyperrectangle.
\end{corollary}

\subsection{Proof Sketch}
Each step of Figure~\ref{fig:roadmap} is one lemma, and the final paragraph combines them. All proofs
of this section are given in Appendix~\ref{sec:proofs_functional}.

\paragraph{Supremum norm to $L^{2}$ norm.}
For a polynomial activation the sample paths of $f$ and of $G$ lie in a space of polynomials, on which the supremum norm and the $L^{2}$ norm are comparable.

\begin{lemma}[norm equivalence for polynomial activations]\label{lem:norm-equiv}
Let $I=I_1\times\cdots\times I_d$ be a hyperrectangle, let $r\in\mathbb N_0$, and let
$\mathcal P_r$ denote the space of real polynomials on $I$ of total degree at most $r$.
For every $p\in\mathcal P_r$,
\begin{equation}\label{eq:norm-equiv}
  \norm{p}_{\mathcal C^{0}(I)}\;\le\;\frac{(r+1)^{d}}{\sqrt{\Vol(I)}}\,\norm{p}_{L^{2}(I)}.
\end{equation}
If $\phi$ is a polynomial of degree $r$, then $f$ has sample paths in $\mathcal P_r$
almost surely, and $G$ admits a version with the same property.
\end{lemma}

\paragraph{Reduction to the grid (Discretization).}

The $L^{2}(I)$ norm is controlled on the midpoint grid of Definition~\ref{def:grid}. The
grid constrains $h$ at finitely many points only, so it says nothing about $h$ in between unless $h$ is
prevented from moving sharply. The H\"older seminorm measures that movement.

\begin{definition}[H\"older spaces]\label{def:holder}
Let $S\subset\R^{d}$ be compact and let $\gamma\in(0,1]$. For $h:S\to\R$ set
\[
  \norm{h}_{\mathcal C^{0}(S)}=\sup_{x\in S}\abs{h(x)},
  \qquad
  [h]_{\mathcal C^{\gamma}(S)}=\sup_{\substack{x,y\in S\\ x\neq y}}
    \frac{\abs{h(x)-h(y)}}{\norm{x-y}^{\gamma}},
  \qquad
  \norm{h}_{\mathcal C^{\gamma}(S)}
    =\norm{h}_{\mathcal C^{0}(S)}+[h]_{\mathcal C^{\gamma}(S)} .
\]
The space $\mathcal C^{\gamma}(S)$ consists of all $h:S\to\R$ with
$\norm{h}_{\mathcal C^{\gamma}(S)}<\infty$.
\end{definition}

A finite $[h]_{\mathcal C^{\gamma}(S)}$ means that $h$ changes by at most
$[h]_{\mathcal C^{\gamma}(S)}\norm{x-y}^{\gamma}$ between any two points of $S$. The
seminorms of $f$ and of $G$ have finite second moment at every fixed width, by the
regularity of $\phi$ (Assumption~\ref{ass:activation}) for $f$ and by the Kolmogorov--Chentsov theorem \citep{RevuzYor1999} for $G$. Both are verified in Appendix~\ref{sec:proofs_functional}, as part of the proof of
Theorem~\ref{thm:functional}.

The midpoint grid has the covering property: every point of $I$ lies within a distance
$\rho_k=O(k^{-1})$ of a grid point (Remark~\ref{lem:grid}). A H\"older function therefore
differs little from its value at the nearest grid point, and comparing $h$ with that value
bounds the $L^{2}(I)$ norm by the grid values, up to an error governed by $\rho_k$ and by
the H\"older seminorm of $h$.

\begin{lemma}[Discretization]
\label{lem:disc}
Let $\gamma \in (0,1]$, let $h \in \mathcal{C}^{\gamma}(I_1\times\cdots\times I_d)$
and let $\inputs_k$ be the midpoint grid of order $k$ on
$I_1\times\cdots\times I_d$. Using $h[\inputs_k] = (h(x_1),\dots,h(x_{k^d}))$
we get
\begin{equation}
  \label{eq:disc}
  \norm{h}^2_{L^2(I)}
  \;\leq\; 2\,\frac{\Vol(I)}{k^d}\,\norm{h[\inputs_k]}^2
  \;+\; 2\,\Vol(I)\,\rho_k^{2\gamma}\,[h]^2_{\mathcal{C}^{\gamma}(I)} .
\end{equation}
\end{lemma}

The proof partitions $I$ into the
$k^{d}$ cells of the grid, each of volume $\Vol(I)/k^{d}$ and each containing its grid
point as midpoint, and compares $h$ with that value on each cell. Two properties of \eqref{eq:disc} matter. The mesh error ($ 2\,\Vol(I)\,\rho_k^{2\gamma}\,[h]^2_{\mathcal{C}^{\gamma}(I)}$) vanishes as the grid is refined,
since $\rho_k=O(k^{-1})$. The first term $2\Vol(I)\norm{h[\inputs_k]}_{n}^{2}$ carries the grid values in the rescaled norm of Definition~\ref{def:w2},
which averages over the grid instead of summing. This is the norm in which
Equation~\eqref{eq:uniform} is uniform in the number of inputs, so refining the grid
shrinks the mesh error without inflating the first term. The two limits of $m$ and $n$ therefore decouple, and this is the
step for which Section~\ref{sec:input_dependence} is needed.

\paragraph{Coupling}
The grid term of \eqref{eq:disc} is a second moment, not a Wasserstein distance. The two
agree when the grid evaluations are optimally coupled, and such a coupling of the full
processes exists.

\begin{lemma}[Coupling the Processes]\label{lem:coupling}
Let $\inputs\subset I$ be a finite sample set with $n=\abs{\inputs}$. There is a pair $(\tilde f,\tilde G)$ of random elements of
$\mathcal C^{0}(I)$ with $\tilde f\overset{d}{=}f$ and $\tilde G\overset{d}{=}G$, whose
evaluations $\bigl(\tilde f[\inputs],\tilde G[\inputs]\bigr)$ form an optimal coupling of
the laws of $f[\inputs]$ and $G[\inputs]$. Consequently
\[
  \mathcal W_{2,n}\bigl(f[\inputs],G[\inputs]\bigr)^{2}
  \;=\;\E\Bigl[\bigl\lVert\tilde f[\inputs]-\tilde G[\inputs]\bigr\rVert_{n}^{2}\Bigr].
\]
\end{lemma}

An optimal coupling of $f[\inputs]$ and $G[\inputs]$ exists by
\citet[Theorem~4.1]{Villani2008OptimalTO}, and the proof extends it from these $n$ values to
the whole functions by drawing $f$ and $G$ independently from their conditional laws given
their values on $\inputs$, which leaves the law of each unchanged.

\paragraph{Completing proof of Theorem~\ref{thm:functional}.} 

We can now put the pieces together. Fix a midpoint grid $\inputs_k$ and let $(\tilde f,\tilde G)$ be the
coupled pair of Lemma~\ref{lem:coupling}, so that $h=f-G$ is a single random function.
Lemma~\ref{lem:norm-equiv} bounds $\norm{h}_{\mathcal C^{0}(I)}$ by $\norm{h}_{L^{2}(I)}$,
and Lemma~\ref{lem:disc} bounds $\norm{h}_{L^{2}(I)}$ by the values on the grid plus a
mesh error. Taking expectations, Lemma~\ref{lem:coupling} turns the grid term into
$\mathcal W_{2,n}\bigl(f[\inputs_k],G[\inputs_k]\bigr)^{2}$, which
Equation~\ref{eq:uniform} bounds by $D^{2}/m$, the same bound for every $k$, its
effective-rank hypothesis supplied by Lemma~\ref{lem:polynomial-er}. The mesh error, by
contrast, vanishes as the grid is refined, since $\rho_k=O(k^{-1})$ while the H\"older seminorms of $f$ and of $G$ have
finite second moment at the fixed width $m$ and do not depend on $k$. Letting $k\to\infty$
at fixed $m$ therefore leaves only the width.

\section{Discussion}
\label{sec:discussion}
Finite-sample bounds have two limitations: their constant depends on the sample set, so a
guarantee transferred from the infinite-width limit weakens as the sample grows; and they
compare two random vectors, not the two random functions.
Section~\ref{sec:input_dependence} shows a simple way to remove the first by measuring the error per input, and
Section~\ref{sec:functional} shows how to remove the second by refining the sample set until it is dense in the domain. For polynomial activations, the result is a functional bound in the quadratic Wasserstein
distance with respect to the supremum norm, at the finite-sample rate $m^{-1/2}$, with an
explicit constant, on every compact set of inputs.

\paragraph{Beyond Polynomial Activations.} The effective-rank condition of Assumption~\ref{ass:eff-rank} holds for
polynomial activations but fails whenever $K[\mathcal{X}]$ has unbounded rank,
in particular for strictly positive definite kernels. For such activations it
must be replaced by a weaker spectral lower bound
$\lambda^+(K[\mathcal{X}]) \geq \nu\, n^{1-\alpha}$ for some $\alpha \geq 0$,
with $\alpha = 0$ recovering Assumption~\ref{ass:eff-rank}. Every step of the roadmap in 
Figure~\ref{fig:roadmap} other than the norm equivalence of
Lemma~\ref{lem:norm-equiv} requires only Assumption~\ref{ass:activation},
and Lemma~\ref{lem:norm-equiv} can be replaced by an interpolation inequality.

\begin{lemma}[interpolation]\label{lem:interp}
Let $\gamma\in(0,1]$ and let $I=I_1\times\cdots\times I_d$ be a hyperrectangle. There are
constants $C_1,C_2<\infty$ depending only on $d$ and on
$\min_j\abs{I_j}$ such that for every $h\in\mathcal C^{\gamma}(I)$,
\[
  \norm{h}_{\mathcal C^{0}(I)}
  \;\le\;C_1\,\norm{h}_{L^{2}(I)}^{\frac{2\gamma}{2\gamma+d}}\,
        [h]_{\mathcal C^{\gamma}(I)}^{\frac{d}{2\gamma+d}}
  \;+\;\frac{C_2}{\sqrt{\Vol(I)}}\,\norm{h}_{L^{2}(I)} .
  \]
\end{lemma}

The proof of Lemma~\ref{lem:interp} is in Appendix~\ref{sec:general-activation}. The $L^{2}(I)$ norm
carries the exponent $\tfrac{2\gamma}{2\gamma+d}$, which is smaller than one and decreases
with the dimension, and the H\"older seminorm carries the remainder. The
seminorm now enters the rate, so it must be bounded uniformly in $m$.

With $\alpha>0$, the right-hand side of \eqref{eq:uniform} is no longer uniform in
the number of inputs, and the grid can no longer be refined freely after
the discretization step. The grid term now grows with $k$ while the mesh error decays, so instead of letting $k \to \infty$ one has to choose $k$ as a function
of $m$ to balance the two. The resulting rate is slower than $m^{-1/2}$, and it
deteriorates with the dimension $d$ through the exponent
$2\gamma/(2\gamma+d)$ of the interpolation inequality.

\paragraph{Sharpness of the rate.} Whether the rate in Theorem~\ref{thm:functional} is sharp remains open. A lower bound does not require a further lifting argument: for every \(x\in I \),
\[
\Wtwo\bigl(f(x),G(x)\bigr)
\le \Wtwo^{\mathcal C^{0}(I)}(f,G).
\]
Combining this inequality with single-input bound of \citet[Theorem~3.3(2)]{DBLP:journals/corr/abs-2307-06092}  yields a lower bound of order \(m^{-1}\). Since this bound uses only one input, it does not establish whether the rate $m^{-1/2}$ is sharp for general \(n\). The upper bound in Theorem~\ref{thm:functional}, in contrast, inherits the rate of the corresponding finite-input bound. For finite input sets, a rate of \(m^{-1}\) is available when \(K[\inputs]\) is invertible \citep{Trevisan2023WideDN}. When it is singular, the best known rate is \(m^{-1/2}\). Our setting falls into the latter case: the effective-rank condition imposes a rank bound on \(K[\inputs]\), precluding invertibility for sufficiently large input sets. Thus, the available finite-input bounds yield the rate \(m^{-1/2}\).

\paragraph{Deeper networks.} Our analysis covers only one hidden layer. The steps in
Section~\ref{sec:functional} that convert a finite-sample bound into a functional one carry
over to deeper networks with polynomial activations. The finite-sample bound of
\citet[Theorem~3.1]{DBLP:journals/ml/BasteriT24} also holds at any depth, but its proof
propagates the comparison through the layers and therefore requires a Lipschitz
activation, while our functional argument requires a polynomial one. Extending Theorem~\ref{thm:functional} to deeper networks therefore requires a propagation
step that does not rely on Lipschitz continuity.

\paragraph{The metric.} Our bound controls the function values but not their derivatives.
The only stronger metric in Table~\ref{tab:related} is the $\mathcal C^{k}$ norm of
\citet{DBLP:journals/corr/abs-2307-06092}, obtained at the slower rate $m^{-1/8}$, and
whether $m^{-1/2}$ is attainable in such a norm is open in general.

\paragraph{Implications for tensor programs.} Tensor programs \citep{DBLP:journals/corr/abs-1910-12478} extend the Gaussian process limit to essentially all standard architectures, and \citet{DBLP:journals/corr/abs-2607-06290} recently made this limit quantitative, recovering the $m^{-1/2}$ rate of \citet{DBLP:journals/ml/BasteriT24} for general programs. The bounds of \citet{DBLP:journals/corr/abs-2607-06290} are finite-sample statements: they hold on a fixed set of inputs, with a constant that depends implicitly on that set. Our approach suggests how to lift such bounds to functional statements for tensor programs, yielding convergence rates that do not deteriorate as the evaluation set grows, provided their dependence on the number of inputs can be made explicit.
\section*{AI Disclosure}
ChatGPT (OpenAI) and Claude (Anthropic) were used as auxiliary tools for language editing, literature discovery, and typesetting support. Each reference and bibliographic item identified through these tools was independently verified against primary sources, and we take full responsibility for the final content and accuracy of the document.


\bibliography{bib}

@Inbook{Neal1996,
author="Neal, Radford M.",
title="Priors for Infinite Networks",
bookTitle="Bayesian Learning for Neural Networks",
year="1996",
publisher="Springer New York",
pages="29--53",
}

@inproceedings{DBLP:conf/nips/Williams96,
  author       = {Christopher K. I. Williams},
  title        = {Computing with Infinite Networks},
  booktitle    = {Advances in Neural Information Processing Systems },
  pages        = {295--301},
  publisher    = {{MIT} Press},
  year         = {1996},
}

@inproceedings{DBLP:conf/iclr/LeeBNSPS18,
  author       = {Jaehoon Lee and
                  Yasaman Bahri and
                  Roman Novak and
                  Samuel S. Schoenholz and
                  Jeffrey Pennington and
                  Jascha Sohl{-}Dickstein},
  title        = {Deep Neural Networks as {G}aussian Processes},
  booktitle    = {nternational Conference on Learning Representations},
  year         = {2018},
}

@inproceedings{DBLP:conf/iclr/MatthewsHRTG18,
  author       = {Matthews, Alexander G. de G.  and
                  Jiri Hron and
                  Mark Rowland and
                  Richard E. Turner and
                  Zoubin Ghahramani},
  title        = {{G}aussian Process Behaviour in Wide Deep Neural Networks},
  booktitle    = {International Conference on Learning Representations},
  year         = {2018},
}

@article{DBLP:journals/corr/abs-1910-12478,
  author       = {Greg Yang},
  title        = {Tensor Programs {I:} Wide Feedforward or Recurrent Neural Networks
                  of Any Architecture are {G}aussian Processes},
  journal      = {CoRR},
  year         = {2019},
}

@inproceedings{DBLP:conf/colt/EldanMS21,
  author       = {Ronen Eldan and
                  Dan Mikulincer and
                  Tselil Schramm},
  title        = {Non-asymptotic approximations of neural networks by {G}aussian processes},
  booktitle    = {Conference on Learning Theory},
  series       = {Proceedings of Machine Learning Research},
  volume       = {134},
  pages        = {1754--1775},
  year         = {2021},
}

@inproceedings{DBLP:conf/iclr/BracaleFFP21,
  author       = {Daniele Bracale and
                  Stefano Favaro and
                  Sandra Fortini and
                  Stefano Peluchetti},
  title        = {Large-width functional asymptotics for deep {G}aussian neural networks},
  booktitle    = {International Conference on Learning Representations},
  year         = {2021},
}

@article{DBLP:journals/corr/abs-2107-01562,
  author       = {Boris Hanin},
  title        = {Random Neural Networks in the Infinite Width Limit as {G}aussian Processes},
  journal      = {CoRR},
  year         = {2021},
}

@article{DBLP:journals/ml/BasteriT24,
  author       = {Andrea Basteri and
                  Dario Trevisan},
  title        = {Quantitative {G}aussian approximation of randomly initialized deep neural
                  networks},
  journal      = {Machine Learning},
  volume       = {113},
  number       = {9},
  pages        = {6373--6393},
  year         = {2024},

}

@inproceedings{Trevisan2023WideDN,
  title={Wide Deep Neural Networks with {G}aussian Weights are Very Close to {G}aussian Processes},
  author={Dario Trevisan},
journal      = {CoRR},
  year={2023},

}

@article{DBLP:journals/corr/abs-2307-06092,
title = {Quantitative {CLT}s in deep neural networks},
author = {S. Favaro and B. Hanin and D. Marinucci and I. Nourdin and G. Peccati},
year = {2025},
volume = {191},
pages = {933--977},
journal = {Probability Theory and Related Fields},
number = {3},
}

@inproceedings{DBLP:conf/aabi/BordinoFF24,
  author       = {Alberto Bordino and
                  Stefano Favaro and
                  Sandra Fortini},
  title        = {Non-asymptotic approximations of {G}aussian neural networks via second-order
                  Poincar{\'{e}} inequalities},
  booktitle    = {Symposium on Advances in Approximate Bayesian Inference},
  series       = {Proceedings of Machine Learning Research},
  volume       = {253},
  pages        = {45--78},
  year         = {2024},
}

@article{DBLP:journals/corr/abs-2504-08335,
  author       = {Lucia Celli and
                  Giovanni Peccati},
  title        = {Entropic bounds for conditionally {G}aussian vectors and applications
                  to neural networks},
  journal      = {CoRR},
  year         = {2025},
}

@article{DBLP:journals/corr/abs-2507-12686,
  author       = {Krishnakumar Balasubramanian and
                  Nathan Ross},
  title        = {Finite-Dimensional {G}aussian Approximation for Deep Neural Networks:
                  Universality in Random Weights},
  journal      = {CoRR},
  year         = {2025},
}

@article{DBLP:journals/corr/abs-2306-16308,
  author       = {Krishnakumar Balasubramanian and
                  Larry Goldstein and
                  Nathan Ross and
                  Adil Salim},
  title        = {{G}aussian random field approximation via {S}tein's method with applications
                  to wide random neural networks},
  journal      = {CoRR},
  year         = {2023},
}

@inproceedings{3844975.3844987,
author = {Ayday, Nil and Yang, Lingchu and Ghoshdastidar, Debarghya},
title = {{G}aussian process limit reveals structural benefits of graph transformers},
year = {2026},
booktitle = {Conference on Uncertainty in Artificial Intelligence},

}

@book{kallenberg2002foundations,
  author = {Kallenberg, Olav},
  edition = {Second},
  publisher = {Springer-Verlag, New York},
  series = {Probability and its Applications },
  title = {Foundations of modern probability},
  year = 2002
}

@inproceedings{DBLP:conf/colt/Klukowski22,
  author       = {Adam Klukowski},
  title        = {Rate of Convergence of Polynomial Networks to {G}aussian Processes},
  booktitle    = {Conference on Learning Theory},
  series       = {Proceedings of Machine Learning Research},
  pages        = {701--722},
  year         = {2022},

}

@inproceedings{DBLP:conf/icml/HronBSN20,
  author       = {Jiri Hron and
                  Yasaman Bahri and
                  Jascha Sohl{-}Dickstein and
                  Roman Novak},
  title        = {Infinite attention: {NNGP} and {NTK} for deep attention networks},
  booktitle    = {International Conference on Machine Learning},
  series       = {Proceedings of Machine Learning Research},
  volume       = {119},
  pages        = {4376--4386},
  year         = {2020},
}

@inproceedings{DBLP:conf/nips/NguyenNXL21,
  author       = {Timothy Nguyen and
                  Roman Novak and
                  Lechao Xiao and
                  Jaehoon Lee},
  title        = {Dataset Distillation with Infinitely Wide Convolutional Networks},
  booktitle    = {Neural Information Processing Systems},
  pages        = {5186--5198},
  year         = {2021},

}

@inproceedings{DBLP:conf/iclr/ChenGW21,
  author       = {Wuyang Chen and
                  Xinyu Gong and
                  Zhangyang Wang},
  title        = {Neural Architecture Search on ImageNet in Four {GPU} Hours: {A} Theoretically
                  Inspired Perspective},
  booktitle    = {International Conference on Learning Representations},
  year         = {2021},

}

@inproceedings{DBLP:conf/iclr/NiuA023,
  author       = {Zehao Niu and
                  Mihai Anitescu and
                  Jie Chen},
  title        = {Graph Neural Network-Inspired Kernels for {G}aussian Processes in Semi-Supervised
                  Learning},
  booktitle    = {International Conference on Learning Representations},
  year         = {2023},
}

@inproceedings{DBLP:conf/icml/AroraDHLW19,
  author       = {Sanjeev Arora and
                  Simon S. Du and
                  Wei Hu and
                  Zhiyuan Li and
                  Ruosong Wang},
  title        = {Fine-Grained Analysis of Optimization and Generalization for Overparameterized
                  Two-Layer Neural Networks},
  booktitle    = {International Conference on Machine Learning},
  series       = {Proceedings of Machine Learning Research},
  volume       = {97},
  pages        = {322--332},
  year         = {2019},
}

@Inbook{Nirenberg2011,
author={Nirenberg, L.},
editor={Faedo, S.},
title={On Elliptic Partial Differential Equations},
bookTitle={Il principio di minimo e sue applicazioni alle equazioni funzionali},
year={2011},
publisher={Springer},
}

@inproceedings{Villani2008OptimalTO,
  title={Optimal Transport: Old and New},
  author={C{\'e}dric Villani},
  year={2008},
}

@article{DBLP:journals/corr/abs-2203-03466,
  author       = {Greg Yang and
                  Edward J. Hu and
                  Igor Babuschkin and
                  Szymon Sidor and
                  Xiaodong Liu and
                  David Farhi and
                  Nick Ryder and
                  Jakub Pachocki and
                  Weizhu Chen and
                  Jianfeng Gao},
  title        = {Tensor Programs {V:} Tuning Large Neural Networks via Zero-Shot Hyperparameter
                  Transfer},
  journal      = {CoRR},
  year         = {2022},
}

@inproceedings{DBLP:conf/icml/YangH21,
  author       = {Greg Yang and
                  Edward J. Hu},
  title        = {Tensor Programs {IV:} Feature Learning in Infinite-Width Neural Networks},
  booktitle    = {International Conference on Machine Learning},
  series       = {Proceedings of Machine Learning Research},
  volume       = {139},
  pages        = {11727--11737},
  year         = {2021},
}

@book{RevuzYor1999,
  author = {Revuz, Daniel and Yor, Marc},
  title = {Continuous Martingales and {B}rownian Motion},
  edition = {3},
  series = {Grundlehren der mathematischen Wissenschaften},
  volume = {293},
  publisher = {Springer},
  year = {1999},
}

@article{DBLP:journals/corr/abs-2607-06290,
  author       = {Andrea Agazzi and
                  Eloy M{\'{o}}sig Garc{\'{\i}}a and
                  Dario Trevisan},
  title        = {Quantitative {G}aussian-Process limits of {T}ensor {P}rograms},
  journal      = {CoRR},
  year         = {2026},
}
\bibliographystyle{plainnat}
\appendix


\section{Auxiliary results}
\label{app:w2}
\begin{proposition}[properties of $\Wtwo$]\label{prop:w2}
Let $X,Y,Z$ be $E$-valued random variables with finite second moments.
\begin{enumerate}
  \item\label{it:coupling} \textup{(coupling bound)} For any pair $(X,Y)$ on a
    common probability space, \[\Wtwo^{E}(X,Y)\le\sqrt{\E[\norm{X-Y}_E^{2}]}\].

  \item\label{it:triangle} \textup{(triangle inequality)}
    $\Wtwo^{E}(X,Z)\le\Wtwo^{E}(X,Y)+\Wtwo^{E}(Y,Z)$.

  \item\label{it:translation} \textup{(independent translations)} If $Z$ is
    independent of $X$ and of $Y$, then
    \[ \Wtwo^{E}(X+Z,Y+Z)\le\Wtwo^{E}(X,Y). \]

  \item\label{it:convexity} \textup{(convexity)} If $Z$ takes values in some Euclidean space,
    then
    \[ \bigl(\Wtwo^{E}(X,Y)\bigr)^{2} \le\int\bigl(\Wtwo^{E}(\Prob_{X\mid Z=z},\Prob_Y)\bigr)^{2}\,d\Prob_Z(z). \]

  \item\label{it:gaussian} \textup{(Gaussian comparison)} If $E=\R^{n}$ and
    $X,Y$ are centred Gaussian, then
    \[ \Wtwo(X,Y)\le\norm{\sqrt{\Sigma(X)}-\sqrt{\Sigma(Y)}}. \]
\end{enumerate}
\end{proposition}

\begin{remark}[properties of the midpoint grids]\label{lem:grid}
\noindent \textbf{1.} $\abs{\inputs_k}=k^{d}$: the subscript $k$ counts points per
axis, not the total number of points.\\

\noindent \textbf{2.} For every $x\in I$ there is $x_i\in\inputs_k$ with
  \begin{equation}\label{eq:covering}
    \norm{x-x_i} \;\le\; \rho_k
    := \frac{1}{2k}\Bigl(\sum_{j=1}^{d}\abs{I_j}^{2}\Bigr)^{1/2} .
  \end{equation}
  Choosing $x_i$ coordinatewise gives $\abs{x_j-(x_i)_j}\le\abs{I_j}/(2k)$ for
  each $j$ leading to \eqref{eq:covering}.
\end{remark}

\begin{lemma}[square-root perturbation]\label{lem:sqrt}
Let $A,B\in\R^{n\times n}$ be positive semi-definite with $\range(A)\subseteq\range(B)$.
Then
\[
  \bigl\lVert\sqrt A-\sqrt B\bigr\rVert
  \;\le\;\frac{\norm{A-B}}{\sqrt{\lambda^{+}(B)}} .
\]
\end{lemma}

\begin{proof}
Put $W:=\range(B)$. Since $A$ and $B$ are symmetric with $\range(A)\subseteq W$, both
vanish on $W^{\perp}$ and map $W$ into itself, and so do $\sqrt A$ and $\sqrt B$. On $W$
the matrix $B$ is positive definite with smallest eigenvalue $\lambda^{+}(B)$, and the
standard bound
$\lVert\sqrt A-\sqrt B\rVert\le\bigl(\sqrt{\lambda_{\min}(A)}+\sqrt{\lambda_{\min}(B)}\bigr)^{-1}\norm{A-B}$
for positive definite matrices applies there. All four matrices vanish on $W^{\perp}$, so
the restriction changes no norm.
\end{proof}

\begin{lemma}[range of a second-moment matrix]\label{lem:range}
Let $V$ be a random vector in $\R^{n}$ with $\E\norm{V}^{2}<\infty$ and
$\Sigma:=\E\bigl[VV^{\top}\bigr]$. Then $V\in\range(\Sigma)$ almost surely.
\end{lemma}

\begin{proof}
If $\Sigma u=0$ then $\E\bigl[(u^{\top}V)^{2}\bigr]=u^{\top}\Sigma u=0$, so
$u^{\top}V=0$ almost surely. Applying this to a basis of $\ker\Sigma$ gives
$V\perp\ker\Sigma=\range(\Sigma)^{\perp}$ almost surely.
\end{proof}

\section{Proofs of Section~\ref{sec:prelim}}
\label{sec:proofs_prelim}

\begin{proof}[Theorem~\ref{thm:bt}]
Write $\mathbf 1\in\R^{n}$ for the all-ones vector and let $Z^{(r)}\in\R^{n}$ denote the
$r$-th row of $g[\inputs]\in\R^{m\times n}$, that is
$Z^{(r)}=\bigl(g_r(x_1),\dots,g_r(x_n)\bigr)^{\top}$.

\medskip
\emph{Step 1 (structure of the evaluation vector).}
Each $g_r(x)=\langle W_r,x\rangle+b_r$ is an affine function of $x$ with independent
centred Gaussian coefficients, so $Z^{(r)}$ is a centred Gaussian vector with
\[
  \Cov\bigl(g_r(x_i),g_r(x_j)\bigr)
  =\frac{\sigma_w^{2}}{d}\langle x_i,x_j\rangle+\sigma_b^{2}
  =\kappa(x_i,x_j)
\]
by \eqref{eq:kappa}; the pairs $(W_r,b_r)$ are independent across $r$, so
$Z^{(1)},\dots,Z^{(m)}$ are i.i.d.\ $\Normal\bigl(0,\kappa[\inputs]\bigr)$. The input
dimension enters only here, through $\kappa$. With $\phi$ acting componentwise,
\[
  f[\inputs]\;=\;\sum_{r=1}^{m}a_r\,\phi\bigl(Z^{(r)}\bigr)\;+\;c\,\mathbf 1 .
\]
By Assumption~\ref{ass:activation}\,(\ref{it:growth}) the vector $\phi(Z^{(r)})$ has
finite moments of every order, so all quantities below are finite.

\medskip
\emph{Step 2 (removing the output bias).}
Since $K_c=K+\tau_c^{2}$, we may realise $G[\inputs]=\Gamma+c\,\mathbf 1$ with
$\Gamma\sim\Normal\bigl(0,K[\inputs]\bigr)$ independent of $c$, and $c$ is independent of
$(W,b,a)$. Writing $\tilde f[\inputs]:=\sum_r a_r\phi(Z^{(r)})$, the independent
translation $c\,\mathbf 1$ cancels by Proposition~\ref{prop:w2}\,(\ref{it:translation}):
\[
  \Wtwo\bigl(f[\inputs],G[\inputs]\bigr)\;\le\;\Wtwo\bigl(\tilde f[\inputs],\Gamma\bigr).
\]

\medskip
\emph{Step 3 (conditional Gaussianity).}
Condition on $(W,b)$, equivalently on $Z^{(1)},\dots,Z^{(m)}$. The coefficients $a_r$ are
i.i.d.\ $\Normal(0,\tau_a^{2}/m)$ and independent of $(W,b)$, so conditionally
$\tilde f[\inputs]$ is a centred Gaussian vector with covariance
\[
  \widehat K_m\;:=\;\frac{\tau_a^{2}}{m}\sum_{r=1}^{m}
    \phi\bigl(Z^{(r)}\bigr)\phi\bigl(Z^{(r)}\bigr)^{\top},
  \qquad
  \E\bigl[\widehat K_m\bigr]=\tau_a^{2}\,\E\bigl[\phi(Z)\phi(Z)^{\top}\bigr]=K[\inputs],
\]
the last identity being \eqref{eq:K}. By
Proposition~\ref{prop:w2}\,(\ref{it:convexity}), applied with the Euclidean conditioning
variable $(W,b)$,
\[
  \Wtwo\bigl(\tilde f[\inputs],\Gamma\bigr)^{2}
  \;\le\;\E\Bigl[\Wtwo\bigl(\Normal(0,\widehat K_m),\,
                            \Normal(0,K[\inputs])\bigr)^{2}\Bigr].
\]

\medskip
\emph{Step 4 (covariance perturbation).}
Apply Lemma~\ref{lem:range} to $V=\tau_a\phi(Z)$, whose second-moment matrix is
$K[\inputs]$: almost surely $\phi(Z)\in\range\bigl(K[\inputs]\bigr)$. Each $Z^{(r)}$ has
the law of $Z$, so almost surely
$\range(\widehat K_m)\subseteq\range\bigl(K[\inputs]\bigr)$. Assumption
\ref{ass:activation}\,(\ref{it:nondeg}) gives $K(x,x)>0$, so $K[\inputs]\neq0$ and
$\lambda^{+}(K[\inputs])$ is defined. Proposition~\ref{prop:w2}\,(\ref{it:gaussian})
followed by Lemma~\ref{lem:sqrt} therefore yields
\[
  \Wtwo\bigl(\Normal(0,\widehat K_m),\Normal(0,K[\inputs])\bigr)
  \;\le\;\bigl\lVert\sqrt{\widehat K_m}-\sqrt{K[\inputs]}\bigr\rVert
  \;\le\;\frac{\bigl\lVert\widehat K_m-K[\inputs]\bigr\rVert}
              {\sqrt{\lambda^{+}\bigl(K[\inputs]\bigr)}} .
\]

\medskip
\emph{Step 5 (the fluctuation of $\widehat K_m$).}
The matrices $Y_r:=\tau_a^{2}\phi(Z^{(r)})\phi(Z^{(r)})^{\top}$ are i.i.d.\ with mean
$K[\inputs]$, and $\widehat K_m-K[\inputs]=\tfrac1m\sum_r\bigl(Y_r-K[\inputs]\bigr)$.
Expanding the Frobenius norm and using that the summands are independent and centred, the
cross terms vanish, so
\[
  \E\bigl\lVert\widehat K_m-K[\inputs]\bigr\rVert^{2}
  \;=\;\frac{1}{m}\,
    \E\Bigl[\bigl\lVert\tau_a^{2}\phi(Z)\phi(Z)^{\top}-K[\inputs]\bigr\rVert^{2}\Bigr].
\]
Combining Steps 2--5 gives
$\Wtwo(f[\inputs],G[\inputs])^{2}\le C(\phi,\inputs)^{2}/m$ with $C(\phi,\inputs)$ as in
\eqref{eq:bt-constant}.
No step above uses global Lipschitz continuity of $\phi$. The only property of the
activation that enters is the finiteness of the fourth moment in Step~5, which follows from
Assumption~\ref{ass:activation}\,(\ref{it:growth}).
\end{proof}

\begin{corollary}[at most linear growth]\label{thm:numerator}
Let $\inputs\subset I$ be a finite sample set with $\abs{\inputs}=n$ and let
$\phi$ satisfy Assumption~\ref{ass:activation} with growth constants $A,p$. Then
the numerator of \eqref{eq:bt-constant},
\[
  N:=\sqrt{\E\Bigl[\bigl\lVert\tau_a^2\,\phi(Z)\phi(Z)^{\top}
      -K[\inputs]\bigr\rVert^{2}\Bigr]},
  \qquad Z\sim\Normal\bigl(0,\kappa[\inputs]\bigr)\in\R^{n},
\]
satisfies $N\le \tau_a^2 A^2\sqrt{M_{4p}}\;n$, where
\[
  M_{s}:=\E\bigl[(1+\abs{\xi})^{s}\bigr]<\infty
  \quad\text{for }s>0 \text{ and } \xi\sim\Normal(0,\kappa_I^{2}) \text{ with } \kappa_I^{2}:=\max_{x\in I}\kappa(x,x)<\infty.
\]

The constant depends only on $\phi$ through $A,p$, on $I$, and on the variances, and not on $n$.
\end{corollary}

\begin{proof}[Corollary~\ref{thm:numerator}]
Abbreviate $Z=(Z_1,\dots,Z_n)\sim\Normal\bigl(0,\kappa[\inputs]\bigr)$ and
consider the random matrix
\[
  Y:=\tau_a^2\,\phi(Z)\phi(Z)^{\top}\in\R^{n\times n},
  \qquad Y_{ij}=\tau_a^2\,\phi(Z_i)\phi(Z_j).
\]
For every pair $(i,j)$ the vector $(Z_i,Z_j)$ is centred Gaussian with
covariance
$\begin{psmallmatrix}\kappa(x_i,x_i)&\kappa(x_i,x_j)\\
\kappa(x_i,x_j)&\kappa(x_j,x_j)\end{psmallmatrix}$,
so by \eqref{eq:K} we have $\E[Y_{ij}]=K(x_i,x_j)$, that is
$\E[Y]=K[\inputs]$.

\emph{Step 1 (centring).} Entrywise,
$\E\bigl[(Y_{ij}-\E[Y_{ij}])^2\bigr]=\E[Y_{ij}^2]-\bigl(\E[Y_{ij}]\bigr)^2
\le\E[Y_{ij}^2]$. Summing over $i,j$ gives
\[
  N^2=\E\Bigl[\bigl\lVert Y-K[\inputs]\bigr\rVert^{2}\Bigr]
  \;\le\;\E\bigl[\norm{Y}^{2}\bigr].
\]

\emph{Step 2 (reduction to fourth moments).} By the definition of the Frobenius
norm,
\[
  \norm{Y}^{2}=\bigl(\tau_a^2\bigr)^{2}\sum_{i,j=1}^{n}\phi(Z_i)^2\phi(Z_j)^2,
  \qquad\text{hence}\qquad
  \E\bigl[\norm{Y}^{2}\bigr]
  =\bigl(\tau_a^2\bigr)^{2}\sum_{i,j=1}^{n}\E\bigl[\phi(Z_i)^2\phi(Z_j)^2\bigr].
\]

\emph{Step 3 (a uniform fourth-moment bound).} By the Cauchy--Schwarz
inequality,
$\E[\phi(Z_i)^2\phi(Z_j)^2]\le\E[\phi(Z_i)^4]^{1/2}\E[\phi(Z_j)^4]^{1/2}$, so it
suffices to bound $\E[\phi(Z_i)^4]$ uniformly in $i$. Polynomial growth
(Assumption~\ref{ass:activation}\,(1)) gives $\phi(z)^4\le A^4(1+\abs{z})^{4p}$.
Each $Z_i$ is centred Gaussian with variance
$\kappa_i^2=\kappa(x_i,x_i)\le\kappa_I^2$, and for $\xi\sim\Normal(0,1)$ the map
$\sigma\mapsto\E\bigl[(1+\sigma\abs{\xi})^{4p}\bigr]$ is non-decreasing on
$[0,\infty)$, whence
\[
  \E\bigl[\phi(Z_i)^4\bigr]\le A^4\,\E\bigl[(1+\abs{Z_i})^{4p}\bigr]
  \le A^4\,\E_{Z\sim\Normal(0,\kappa_I^2)}\bigl[(1+\abs{Z})^{4p}\bigr]
  = A^4 M_{4p}.
\]

Combining the three steps,
\[
  N^{2}\;\le\;\bigl(\tau_a^2\bigr)^{2}\sum_{i,j=1}^{n}A^4M_{4p}
  \;=\;\bigl(\tau_a^2\bigr)^{2}A^4M_{4p}\,n^2 ,
\]
and taking square roots yields $N\le\tau_a^2A^2\sqrt{M_{4p}}\,n$. The moment
$M_{4p}$ is finite because Gaussian random variables have moments of all orders,
and $\kappa_I^2<\infty$ because $I$ is compact and $\kappa$ is continuous; the
constant depends only on $A$, $p$, on $I$ through $\kappa_I^2$, and on the
variances $\sigma_w^2,\sigma_b^2,\tau_a^2$, as stated.
\end{proof}

\begin{proof}[Proposition~\ref{cor:constant}]
Part~(1) is Corollary~\ref{thm:numerator}. The matrix inside the norm in \eqref{eq:bt-constant} has
$n^2$ entries, each bounded in expectation uniformly over the compact set $I$, so its
Frobenius norm grows like $\sqrt{n^{2}}=n$. We bound the centred fluctuation by the second
moment of $\tau_a^2\phi(Z)\phi(Z)^{\top}$, reduce the Frobenius norm to its entries, and
control those by Assumption~\ref{ass:activation}\,(\ref{it:growth}).

For part~(2), let $\inputs\subset I$ satisfy \eqref{eq:eff-rank} with constant $\nu$ and
let $n=\abs{\inputs}$. Then Corollary~\ref{thm:numerator} and \eqref{eq:eff-rank} give
\begin{equation}\label{eq:Cgrowth}
  C(\phi,\inputs)^{2}\;\le\;
  \frac{\bigl(\tau_a^2A^2\sqrt{M_{4p}}\,n\bigr)^{2}}{\nu\,n}\;=\;D^{2}n,
  \qquad D:=\tau_a^2A^2\sqrt{\frac{M_{4p}}{\nu}} ,
\end{equation}
so that $C(\phi,\inputs)\le D\sqrt n$, and Theorem~\ref{thm:bt} gives
\[
  \Wtwo\bigl(f[\inputs],G[\inputs]\bigr)\;\le\;\frac{C(\phi,\inputs)}{\sqrt m}
  \;\le\;D\sqrt{\frac{n}{m}} .
\]

The order $\sqrt n$ cannot be improved within \eqref{eq:bt-constant}. For any finite
sample set $\inputs\subset I$ with $n=\abs{\inputs}$, write
$Z\sim\Normal\bigl(0,\kappa[\inputs]\bigr)\in\R^{n}$ as in
Corollary~\ref{thm:numerator}. The trace of a positive semi-definite matrix dominates each
of its eigenvalues, so by Assumption~\ref{ass:activation}\,(\ref{it:growth}) and
$\kappa(x_i,x_i)\le\kappa_I^{2}$, arguing as in the proof of
Corollary~\ref{thm:numerator},
\begin{equation}\label{eq:trace}
  \lambda^{+}\bigl(K[\inputs]\bigr)\;\le\;\tr K[\inputs]
  \;=\;\tau_a^{2}\sum_{i=1}^{n}\E\bigl[\phi(Z_i)^{2}\bigr]\;\le\;\bar\nu\,n,
  \qquad \bar\nu:=\tau_a^{2}A^{2}M_{2p} ,
\end{equation}
so linear growth is the largest order the denominator can have.

Combining
\eqref{eq:trace} with~\eqref{eq:eff-rank} and
$\operatorname{tr}(K[\mathcal X])\ge\operatorname{rank}(K[\mathcal X])\,\lambda^+(K[\mathcal X])$
gives $\operatorname{rank}(K[\mathcal X])\le\bar\nu/\nu$
so the condition forces the $K[\inputs]$ to have bounded rank.  This excludes every
activation for which this rank is unbounded, in particular every activation whose
kernel $K$ is strictly positive definite on $I$, since then
$\operatorname{rank} (K[\mathcal X]) = |\mathcal X|$. The condition is therefore a
finite-rank condition, and Lemma~\ref{lem:polynomial-er} verifies it for the midpoint grids
with polynomial activations.
\end{proof}

\section{Proofs of Section~\ref{sec:functional}}
\label{sec:proofs_functional}

\begin{proof}[Lemma~\ref{lem:polynomial-er}]

Write $\phi(z)=\sum_{h=0}^{r}\beta_h z^{h}$ with $\beta_r\neq0$, and abbreviate
$R:=\binom{r+d}{d}$.

\medskip
\emph{Step 1 (finite-rank representation of the kernel).}
Let $\xi_0\sim\Normal(0,1)$ and $\xi_1\sim\Normal(0,\mathrm{I}_{d\times d})$ be
independent, and set
\begin{equation}\label{eq:Zx}
  Z_x\;:=\;\frac{\sigma_w}{\sqrt d}\,\xi_1^{\top}x+\sigma_b\,\xi_0,
  \qquad x\in\R^{d}.
\end{equation}
This realisation has the covariance required by \eqref{eq:K}: by independence of
$\xi_0$ and $\xi_1$,
\[
  \Cov(Z_x,Z_y)\;=\;\frac{\sigma_w^{2}}{d}\,\Cov\bigl(\xi_1^{\top}x,\xi_1^{\top}y\bigr)
  +\sigma_b^{2}\,\Var(\xi_0)\;=\;\frac{\sigma_w^{2}}{d}\,\langle x,y\rangle+\sigma_b^{2}
  \;=\;\kappa(x,y).
\]
By the binomial theorem,
\begin{equation}\label{eq:expansion}
  \phi(Z_x)\;=\;\sum_{h=0}^{r}\beta_h Z_x^{h}
  \;=\;\sum_{i=0}^{r}\bigl(\xi_1^{\top}x\bigr)^{i}P_i,
  \qquad
  P_i\;:=\;\Bigl(\frac{\sigma_w^{2}}{d}\Bigr)^{\!i/2}
  \sum_{h=i}^{r}\binom{h}{i}\beta_h\,(\sigma_b\xi_0)^{h-i}\;\in\;\R .
\end{equation}
For $0\le i\le r$ define $\varphi_i:\R^{d}\to\R^{\binom{i+d-1}{d-1}}$ by
\begin{equation}\label{eq:features}
  \varphi_i(x)\;=\;\bigoplus_{\abs{\alpha}=i}\sqrt{\binom{i}{\alpha_1,\dots,\alpha_d}}\;x^{\alpha},
\end{equation}
the dimension $\binom{i+d-1}{d-1}$ being the number of multi-indices of total degree
$i$. Fixing the order of these multi-indices gives a bijection
$e_i:\bigl\{1,\dots,\binom{i+d-1}{d-1}\bigr\}\to\{\alpha\in\mathbb N_0^{d}:\abs{\alpha}=i\}$
between the coordinates of $\varphi_i$ and the multi-indices of degree $i$. In this
notation the multinomial theorem reads
$\bigl(\xi_1^{\top}x\bigr)^{i}=\bigl\langle\varphi_i(\xi_1),\varphi_i(x)\bigr\rangle$,
so \eqref{eq:expansion} becomes
\begin{equation}\label{eq:inner}
  \phi(Z_x)\;=\;\sum_{i=0}^{r}\bigl\langle P_i\,\varphi_i(\xi_1),\,\varphi_i(x)\bigr\rangle
  \;=\;\Bigl\langle\bigoplus_{i=0}^{r}P_i\,\varphi_i(\xi_1),\;
  \bigoplus_{i=0}^{r}\varphi_i(x)\Bigr\rangle .
\end{equation}
The vectors in the last inner product have dimension
$\sum_{i=0}^{r}\binom{i+d-1}{d-1}=R$. Setting
\[
  \Psi(x):=\bigoplus_{i=0}^{r}\varphi_i(x)\in\R^{R},\qquad
  \Phi:=\bigoplus_{i=0}^{r}P_i\,\varphi_i(\xi_1),\qquad
  M:=\tau_a^{2}\,\E\bigl[\Phi\Phi^{\top}\bigr],
\]
we obtain from \eqref{eq:K} and \eqref{eq:inner}
\begin{equation}\label{eq:K-factored}
  K(x,y)\;=\;\tau_a^{2}\,\E\bigl[\phi(Z_x)\phi(Z_y)\bigr]\;=\;\Psi(x)^{\top}M\,\Psi(y).
\end{equation}
The matrix $M$ is positive semi-definite, and $M\neq0$: since $\deg\phi=r$, the sum
defining $P_r$ in \eqref{eq:expansion} has only the term $h=r$, so
$P_r=\beta_r(\sigma_w^{2}/d)^{r/2}\neq0$ is deterministic, and the degree-$r$ diagonal
block of $M$ equals
$\tau_a^{2}\beta_r^{2}(\sigma_w^{2}/d)^{r}\,\E\bigl[\varphi_r(\xi_1)\varphi_r(\xi_1)^{\top}\bigr]\neq0$.
In particular $\lambda^{+}(M)>0$.

Let $V_k\in\R^{k^{d}\times R}$ be the matrix with rows
$\Psi(x_1)^{\top},\dots,\Psi(x_{k^{d}})^{\top}$, where $\inputs_k=\{x_1,\dots,x_{k^{d}}\}$.
Then \eqref{eq:K-factored} gives
\begin{equation}\label{eq:Kk}
  K[\inputs_k]\;=\;V_kMV_k^{\top},
\end{equation}
whence $\rank K[\inputs_k]\le\rank(M)\le R$, which is the first assertion.

\medskip
\emph{Step 2 (passing to a matrix of fixed dimension).}
Put $B:=V_kM^{1/2}$, so that $K[\inputs_k]=BB^{\top}$, and set
\begin{equation}\label{eq:Tk}
  T_k\;:=\;B^{\top}B\;=\;M^{1/2}V_k^{\top}V_kM^{1/2}\;\in\;\R^{R\times R}.
\end{equation}
The matrices $BB^{\top}$ and $B^{\top}B$ have the same nonzero eigenvalues, so
$\lambda^{+}\bigl(K[\inputs_k]\bigr)=\lambda^{+}(T_k)$ whenever $T_k\neq0$. Unlike
$K[\inputs_k]$, the matrix $T_k$ has dimension $R$, independent of $k$. We therefore
analyse the normalised moment matrix
\begin{equation}\label{eq:Theta}
  \Theta_k\;:=\;\frac{1}{k^{d}}V_k^{\top}V_k\;\in\;\R^{R\times R},
  \qquad\text{so that}\qquad
  T_k\;=\;k^{d}\,M^{1/2}\Theta_kM^{1/2}.
\end{equation}

\medskip
\emph{Step 3 (the Riemann-sum limit).}
Index the rows and columns of $\Theta_k$ by pairs $(i,s)$ with $0\le i\le r$ and
$1\le s\le\binom{i+d-1}{d-1}$, where $(i,s)$ refers to the $s$-th coordinate of the
degree-$i$ block of $\Psi$. By \eqref{eq:features},
\begin{equation}\label{eq:Theta-entries}
  \bigl(\Theta_k\bigr)_{(i,s),(j,t)}
  \;=\;\frac{1}{k^{d}}\sum_{l=1}^{k^{d}}\bigl(\varphi_i(x_l)\bigr)_s\bigl(\varphi_j(x_l)\bigr)_t
  \;=\;c_{(i,s),(j,t)}\cdot\frac{1}{k^{d}}\sum_{l=1}^{k^{d}}x_l^{\,e_i(s)+e_j(t)},
\end{equation}
where
$c_{(i,s),(j,t)}:=\bigl[\binom{i}{e_i(s)}\binom{j}{e_j(t)}\bigr]^{1/2}$
does not depend on $k$. Since $\inputs_k$ is the product grid \eqref{eq:grid}, the points
$x_l$ are the midpoints of the $k^{d}$ cells of a uniform partition of $I$, each of volume
$\Vol(I)/k^{d}$. The sum in \eqref{eq:Theta-entries} is therefore the $d$-dimensional
midpoint rule for the continuous function $x\mapsto x^{\alpha}$ with
$\alpha:=e_i(s)+e_j(t)$, and
\begin{equation}\label{eq:riemann}
  \frac{1}{k^{d}}\sum_{l=1}^{k^{d}}x_l^{\alpha}
  \;=\;\frac{1}{\Vol(I)}\sum_{l=1}^{k^{d}}\frac{\Vol(I)}{k^{d}}\,x_l^{\alpha}
  \;\xrightarrow[k\to\infty]{}\;\frac{1}{\Vol(I)}\int_I x^{\alpha}\,dx .
\end{equation}
Writing $H$ for the matrix with entries
$H_{(i,s),(j,t)}:=c_{(i,s),(j,t)}\Vol(I)^{-1}\int_I x^{e_i(s)+e_j(t)}\,dx$, the convergence
\eqref{eq:riemann} holds entrywise, and since the dimension $R$ does not depend on $k$
this gives $\norm{\Theta_k-H}\to0$.

The limit $H$ is positive definite. Indeed, for $u\in\R^{R}$,
\begin{equation}\label{eq:H-quadratic}
  u^{\top}Hu\;=\;\frac{1}{\Vol(I)}\int_I q_u(x)^{2}\,dx,
  \qquad
  q_u(x)\;:=\;\sum_{i=0}^{r}\;\sum_{s=1}^{\binom{i+d-1}{d-1}}u_{(i,s)}
  \sqrt{\binom{i}{e_i(s)}}\;x^{e_i(s)} .
\end{equation}
The map $(i,s)\mapsto e_i(s)$ is a bijection onto $\{\alpha:\abs{\alpha}\le r\}$, and the
monomials $x^{\alpha}$ are linearly independent, so $u\neq0$ implies $q_u\not\equiv0$. A
nonzero polynomial vanishes on a Lebesgue null set, so the integral in
\eqref{eq:H-quadratic} is strictly positive, and $\lambda_{\min}(H)>0$. For every unit
vector $v$,
\[
  v^{\top}\Theta_kv\;\ge\;v^{\top}Hv-\abs{v^{\top}(\Theta_k-H)v}
  \;\ge\;\lambda_{\min}(H)-\norm{\Theta_k-H},
\]
so there is $k_1\in\mathbb N$ with
\begin{equation}\label{eq:lammin}
  \lambda_{\min}(\Theta_k)\;\ge\;\tfrac12\lambda_{\min}(H)\;>\;0
  \qquad\text{for all }k\ge k_1 .
\end{equation}

\medskip
\emph{Step 4 (conclusion).}
Fix $k\ge k_1$. By \eqref{eq:lammin} the matrix $\Theta_k$ is positive definite, so
$T_k\neq0$: if $M^{1/2}\Theta_kM^{1/2}=0$ then
$\norm{\Theta_k^{1/2}M^{1/2}u}^{2}=u^{\top}M^{1/2}\Theta_kM^{1/2}u=0$ for every $u$,
forcing $M^{1/2}=0$ and hence $M=0$, which Step 1 excludes. From \eqref{eq:Tk} we have
$\range(T_k)\subseteq\range(M^{1/2})=\range(M)$, and every unit vector
$v\in\range(M)$ satisfies $v^{\top}Mv\ge\lambda^{+}(M)$. Hence for every unit vector
$v\in\range(T_k)$,
\begin{equation}\label{eq:final-quad}
\begin{split}
  \frac{1}{k^{d}}\,v^{\top}T_kv
  & = \bigl(M^{1/2}v\bigr)^{\top}\Theta_k\bigl(M^{1/2}v\bigr) \\
  & \ge \lambda_{\min}(\Theta_k)\,\norm{M^{1/2}v}^{2} \\
  & = \lambda_{\min}(\Theta_k)\,v^{\top}Mv \\
  & \ge \lambda_{\min}(\Theta_k)\,\lambda^{+}(M).
\end{split}
\end{equation}
Minimising over such $v$ gives
$\lambda^{+}\bigl(K[\inputs_k]\bigr)=\lambda^{+}(T_k)\ge k^{d}\lambda_{\min}(\Theta_k)\lambda^{+}(M)$,
and with \eqref{eq:lammin},
\[
  \lambda^{+}\bigl(K[\inputs_k]\bigr)\;\ge\;\nu\,k^{d}
  \qquad\text{for every }k\ge k_1,
  \qquad \nu:=\tfrac12\lambda_{\min}(H)\,\lambda^{+}(M)\;>\;0 .
\]
For the finitely many $k<k_1$, Assumption~\ref{ass:activation}(3) gives, for every
$x\in\mathcal X_k$, that $\kappa(x,x)>0$ and that $\phi$ does not vanish
$\mathcal N(0,\kappa(x,x))$-almost everywhere, so
$K(x,x) = \tau_a^2\,\mathbb E[\phi(Z_x)^2] > 0$. Hence $K[\mathcal X_k]\neq 0$ and
$\lambda^+(K[\mathcal X_k])>0$; replacing $\nu$ by
$\min\{\nu,\ \min_{k<k_1} k^{-d}\lambda^+(K[\mathcal X_k])\}>0$ gives the bound for
every $k\in\mathbb N$.
\end{proof}

\begin{proof}[Lemma~\ref{lem:norm-equiv}]
\emph{Step 1 (the supremum bound).}
The affine change of variables $x_j=\tfrac{a_j+b_j}{2}+\tfrac{\abs{I_j}}{2}t_j$ maps
$[-1,1]^{d}$ onto $I$ and preserves total degree. Writing $\tilde p(t):=p(x(t))$, we have
$\norm{p}_{\mathcal C^{0}(I)}=\norm{\tilde p}_{\mathcal C^{0}([-1,1]^{d})}$ and
$\norm{p}_{L^{2}(I)}^{2}=2^{-d}\Vol(I)\,\norm{\tilde p}_{L^{2}([-1,1]^{d})}^{2}$, so it
suffices to prove
$\norm{\tilde p}_{\mathcal C^{0}}\le 2^{-d/2}(r+1)^{d}\norm{\tilde p}_{L^{2}}$ on the cube.

Let $P_0,P_1,\dots$ be the Legendre polynomials on $[-1,1]$, so that
$\norm{P_n}_{\infty}=1$ and $\int_{-1}^{1}P_n^{2}=\tfrac{2}{2n+1}$. Since $\tilde p$ has
degree at most $r$ in each coordinate, it expands as
$\tilde p(t)=\sum_{\alpha}c_\alpha\prod_{j=1}^{d}P_{\alpha_j}(t_j)$ over multi-indices
$\alpha$ with $\alpha_j\le r$, and these products are orthogonal on $[-1,1]^{d}$. Hence
\[
  \norm{\tilde p}_{\mathcal C^{0}}\le\sum_{\alpha}\abs{c_\alpha}
  \le\Bigl(\sum_{\alpha}c_\alpha^{2}\prod_{j}\tfrac{2}{2\alpha_j+1}\Bigr)^{1/2}
     \Bigl(\sum_{\alpha}\prod_{j}\tfrac{2\alpha_j+1}{2}\Bigr)^{1/2}
  =\norm{\tilde p}_{L^{2}}
     \Bigl(\prod_{j=1}^{d}\sum_{n=0}^{r}\tfrac{2n+1}{2}\Bigr)^{1/2},
\]
by Cauchy--Schwarz and Parseval. As $\sum_{n=0}^{r}\tfrac{2n+1}{2}=\tfrac{(r+1)^{2}}{2}$,
the last factor equals $2^{-d/2}(r+1)^{d}$, which proves the first inequality
in \eqref{eq:norm-equiv}.

\medskip
\emph{Step 2 (sample paths).}
Each coordinate of $Wx+b$ is affine in $x$, so $\phi$ applied to it is a polynomial of
total degree at most $r$; hence $f(x)=a^{\top}\phi(Wx+b)+c$ lies in $\mathcal P_r$ for
every realisation. For $G$, Step~1 of the proof of Lemma~\ref{lem:polynomial-er} gives
$K(x,y)=\Psi(x)^{\top}M\Psi(y)$ with $\Psi$ a vector of polynomials of degree at most $r$
and $M\succeq0$. Let $\zeta\sim\Normal(0,M)$ and $c'\sim\Normal(0,\tau_c^{2})$ be
independent, and set $\tilde G(x):=\Psi(x)^{\top}\zeta+c'$. Then $\tilde G$ is a centred
Gaussian process with
$\Cov(\tilde G(x),\tilde G(y))=\Psi(x)^{\top}M\Psi(y)+\tau_c^{2}=K_c(x,y)$, so
$\tilde G$ has the law of $G$, and its sample paths lie in $\mathcal P_r$.
\end{proof}
\begin{proof}[Lemma~\ref{lem:disc}]
Partition $I$ into the $k^d$ cells $Q_1,\dots,Q_{k^d}$ of the grid, each of
volume $\Vol(I)/k^d$ and each containing its grid point $x_i$ as midpoint. For
$x \in Q_i$ we have $\norm{x-x_i}_2 \leq \rho_k$ by \eqref{eq:covering}, so
\[
  \abs{h(x)} \leq \abs{h(x_i)} + \abs{h(x)-h(x_i)}
  \leq \abs{h(x_i)} + [h]_{\mathcal{C}^{\gamma}(I)}\norm{x-x_i}_2^{\gamma}
  \leq \abs{h(x_i)} + [h]_{\mathcal{C}^{\gamma}(I)}\rho_k^{\gamma} ,
\]
hence
$h(x)^2 \leq 2h(x_i)^2 + 2\rho_k^{2\gamma}[h]^2_{\mathcal{C}^{\gamma}(I)}$.
Integrating over $Q_i$ and summing over $i$,
\begin{align*}
  \norm{h}^2_{L^2(I)}
  &= \sum_{i=1}^{k^d}\int_{Q_i} h(x)^2\,dx \\
  &\leq \frac{\Vol(I)}{k^d}\sum_{i=1}^{k^d}
    \Bigl(2h(x_i)^2 + 2\rho_k^{2\gamma}[h]^2_{\mathcal{C}^{\gamma}(I)}\Bigr) \\
  &= 2\frac{\Vol(I)}{k^d}\norm{h[\inputs_k]}^2
    + 2\Vol(I)\rho_k^{2\gamma}[h]^2_{\mathcal{C}^{\gamma}(I)} .
\end{align*}
\end{proof}

\begin{proof}[Lemma~\ref{lem:coupling}]
We regard $f$ and $G$ as random elements of the Polish space $\mathcal C^{0}(I)$. The grid evaluations
$f[\inputs]$ and $G[\inputs]$ are random vectors in $\R^{n}$. Let $\pi^{*}$ be an optimal
coupling of their laws for the quadratic Wasserstein distance, which exists by
\citep[Theorem~4.1]{Villani2008OptimalTO}. Since $\mathcal W_{2,n}$ is a positive multiple
of $\Wtwo$, the coupling $\pi^{*}$ is optimal for the rescaled distance as well. Because
$\mathcal C^{0}(I)$ is Polish, the regular conditional distributions
\[
  \mu^{f}_{u}:=\Prob_{f\mid f[\inputs]=u},
  \qquad
  \mu^{G}_{v}:=\Prob_{G\mid G[\inputs]=v},
  \qquad u,v\in\R^{n},
\]
exist, and by disintegration
$1=\Prob\bigl(f[\inputs]=f[\inputs]\bigr)
 =\int_{\R^{n}}\mu^{f}_{u}\bigl(\{\varphi:\varphi[\inputs]=u\}\bigr)\,d\Prob_{f[\inputs]}(u)$,
so $\mu^{f}_{u}$ is concentrated on the fibre $\{\varphi:\varphi[\inputs]=u\}$ for
$\Prob_{f[\inputs]}$-almost every $u$; likewise for $\mu^{G}_{v}$.

Define the measure
\[
  \pi(A):=\int_{\R^{n}\times\R^{n}}\bigl(\mu^{f}_{u}\otimes\mu^{G}_{v}\bigr)(A)\,d\pi^{*}(u,v),
  \qquad A\subseteq\mathcal C^{0}(I)\times\mathcal C^{0}(I)\text{ Borel},
\]
which is a probability measure since $(u,v)\mapsto\mu^{f}_{u}\otimes\mu^{G}_{v}$ is a
probability kernel; on rectangles,
$\pi(A_1\times A_2)=\int\mu^{f}_{u}(A_1)\,\mu^{G}_{v}(A_2)\,d\pi^{*}(u,v)$. The two
coordinate projections on $\bigl(\mathcal C^{0}(I)\times\mathcal C^{0}(I),\pi\bigr)$ define
random functions $\tilde f$ and $\tilde G$. We check the marginals:
\begin{align*}
  \Prob\bigl(\tilde f\in A_1\bigr)
  &=\int_{\R^{n}\times\R^{n}}\mu^{f}_{u}(A_1)\,
     \underbrace{\mu^{G}_{v}\bigl(\mathcal C^{0}(I)\bigr)}_{=1}\,d\pi^{*}(u,v)\\
  &=\int_{\R^{n}}\mu^{f}_{u}(A_1)\,d\Prob_{f[\inputs]}(u)\\
  &=\int_{\R^{n}}\Prob\bigl(f\in A_1\mid f[\inputs]=u\bigr)\,d\Prob_{f[\inputs]}(u)\\
  &=\Prob(f\in A_1),
\end{align*}
where the second equality uses that the first marginal of $\pi^{*}$ is
$\Prob_{f[\inputs]}$, and for $G$ the argument works the same. Moreover, for every Borel
set $A\subseteq\R^{n}\times\R^{n}$,
\begin{align*}
  \Prob\bigl(\bigl(\tilde f[\inputs],\tilde G[\inputs]\bigr)\in A\bigr)
  &=\int_{\R^{n}\times\R^{n}}\bigl(\mu^{f}_{u}\otimes\mu^{G}_{v}\bigr)
     \bigl(\{(\varphi,\psi):(\varphi[\inputs],\psi[\inputs])\in A\}\bigr)\,d\pi^{*}(u,v)\\
  &=\int_{\R^{n}\times\R^{n}}\bigl(\delta_{u}\otimes\delta_{v}\bigr)(A)\,d\pi^{*}(u,v)\\
  &=\int_{\R^{n}\times\R^{n}}\mathbf 1_{A}(u,v)\,d\pi^{*}(u,v)\\
  &=\pi^{*}(A),
\end{align*}
where the second equality uses the concentration on the fibres, whose exceptional sets
are $\pi^{*}$-null because they are null under the marginals of $\pi^{*}$. Therefore
$(\tilde f,\tilde G)$ is a coupling of $f$ and $G$ whose evaluations
$\bigl(\tilde f[\inputs],\tilde G[\inputs]\bigr)$ have the optimal law $\pi^{*}$, and
consequently
\[
  \E\Bigl[\bigl\lVert\tilde f[\inputs]-\tilde G[\inputs]\bigr\rVert_{n}^{2}\Bigr]
  =\int_{\R^{n}\times\R^{n}}\norm{u-v}_{n}^{2}\,d\pi^{*}(u,v)
  =\mathcal W_{2,n}\bigl(f[\inputs],G[\inputs]\bigr)^{2}.
\]
\end{proof}

\begin{proof}[Theorem~\ref{thm:functional}]
Fix $\gamma=\tfrac12$ and $k\ge k_1$, with $k_1$ as in Lemma~\ref{lem:polynomial-er}, and
let $(\tilde f,\tilde G)$ be the coupling of Lemma~\ref{lem:coupling} for $\inputs_k$;
write $h:=\tilde f-\tilde G$, so that $n=\abs{\inputs_k}=k^{d}$.

Since $\phi$ is a polynomial of degree $r$, Lemma~\ref{lem:norm-equiv} gives that $f$ and
$G$ have sample paths in $\mathcal P_r$ almost surely, hence so does $h$, and
\eqref{eq:norm-equiv} may be applied pathwise. Combining it with \eqref{eq:disc} of
Lemma~\ref{lem:disc} and taking expectations,
\begin{equation}
  \label{eq:step1}
  \E\bigl[\norm{h}^{2}_{\mathcal C^{0}(I)}\bigr]
  \;\le\;\frac{(r+1)^{2d}}{\Vol(I)}\,\E\bigl[\norm{h}^{2}_{L^{2}(I)}\bigr]
  \;\le\;2\,(r+1)^{2d}\Bigl(
      \E\bigl[\norm{h[\inputs_k]}_{k^{d}}^{2}\bigr]
      +\rho_k\,\E\bigl[[h]^{2}_{\mathcal C^{1/2}(I)}\bigr]\Bigr),
\end{equation}
the volume cancelling between the two steps.

By Lemma~\ref{lem:coupling} the first expectation equals
$\Wk\bigl(f[\inputs_k],G[\inputs_k]\bigr)^{2}$. The grids $\inputs_k$ satisfy the
effective-rank condition of Assumption~\ref{ass:eff-rank} with a constant $\nu$ that does
not depend on $k$, by Lemma~\ref{lem:polynomial-er}, so Equation~\ref{eq:uniform}
applies with $n=k^{d}$ and gives
\begin{equation}
  \Wk\bigl(f[\inputs_k],G[\inputs_k]\bigr)^{2}\;\le\;\frac{D^{2}}{m},
\end{equation}
uniformly in $k$, with $D$ as in \eqref{eq:Cgrowth}.

For the second expectation, $[\,\cdot\,]_{\mathcal C^{1/2}(I)}$ is a seminorm, so
$[h]^{2}_{\mathcal C^{1/2}(I)}\le 2[\tilde f]^{2}_{\mathcal C^{1/2}(I)}
+2[\tilde G]^{2}_{\mathcal C^{1/2}(I)}$ pathwise. Since $\tilde f\overset{d}{=}f$ and
$\tilde G\overset{d}{=}G$, this gives
\begin{equation}
  \label{eq:Lambda}
  \E\bigl[[h]^{2}_{\mathcal C^{1/2}(I)}\bigr]
  \;\le\;2\,\E\bigl[[f]^{2}_{\mathcal C^{1/2}(I)}\bigr]
        +2\,\E\bigl[[G]^{2}_{\mathcal C^{1/2}(I)}\bigr]
  \;=:\;\Lambda ,
\end{equation}
and $\Lambda$ does not depend on $k$, since it involves only the laws of $f$ and $G$.
Therefore we need to prove that $\Lambda$ stays finite.

To bound the first term, let $W_r$ be the $r$-th row of $W$, so that
$g_r(x)=\langle W_r,x\rangle+b_r$. Since $I$ is convex and $\phi$ is absolutely continuous
with $\abs{\phi'(z)}\le A'(1+\abs{z})^{p'}$
(Assumption~\ref{ass:activation}\,(\ref{it:lipschitz})), each map $x\mapsto\phi(g_r(x))$
is Lipschitz on $I$ with constant
\[
  L_r\;:=\;A'\norm{W_r}\Bigl(1+\max_{x\in I}\abs{g_r(x)}\Bigr)^{p'} .
\]
As $f(x)-f(y)=\sum_{r=1}^{m}a_r\bigl(\phi(g_r(x))-\phi(g_r(y))\bigr)$, for $x,y\in I$
\[
  \abs{f(x)-f(y)}
  \;\le\;\sum_{r=1}^{m}\abs{a_r}L_r\,\norm{x-y}
  \;\le\;\operatorname{diam}(I)^{1/2}\sum_{r=1}^{m}\abs{a_r}L_r\,\norm{x-y}^{1/2},
\]
hence
\[
  \E\bigl[[f]^{2}_{\mathcal C^{1/2}(I)}\bigr]
  \;\le\;\operatorname{diam}(I)\,\E\Bigl[\Bigl(\sum_{r=1}^{m}\abs{a_r}L_r\Bigr)^{2}\Bigr]
  \;<\;\infty ,
\]
the parameters being Gaussian and hence having moments of every order.

The same Lipschitz bound controls the second term. By Step~1 of the proof of
Theorem~\ref{thm:bt}, the pair $\bigl(g_1(x),g_1(y)\bigr)$ has the law of $(Z_x,Z_y)$ in
\eqref{eq:K}, and the constant $\tau_c^{2}$ in \eqref{eq:Kb} cancels, so applying the
bound to the first neuron gives, for $x,y\in I$,
\begin{align*}
  \E\bigl[\bigl(G(x)-G(y)\bigr)^{2}\bigr]
  &=K_c(x,x)-2K_c(x,y)+K_c(y,y)
   =\tau_a^{2}\,\E\bigl[\bigl(\phi(g_1(x))-\phi(g_1(y))\bigr)^{2}\bigr]\\
  &\le\tau_a^{2}\,\E\bigl[L_1^{2}\bigr]\norm{x-y}^{2}\;=:\;L^{2}\norm{x-y}^{2}.
\end{align*}
Since $G(x)-G(y)$ is centred Gaussian, it follows that
\[
  \E\abs{G(x)-G(y)}^{4d}\;\le\;\mu_{4d}L^{4d}\norm{x-y}^{d+3d},
  \qquad \mu_{4d}:=\E\abs{\xi}^{4d},\ \xi\sim\Normal(0,1).
\]
The Kolmogorov--Chentsov theorem in the form of \citet[Theorem~I.2.1]{RevuzYor1999} covers
H\"older exponents below $3d/4d=\tfrac34$, in particular $\tfrac12$, and yields
$\E\bigl[[G]^{4d}_{\mathcal C^{1/2}(I)}\bigr]<\infty$ for the continuous version of $G$
from Lemma~\ref{lem:norm-equiv}, with which the modification provided by the theorem is
indistinguishable. By Jensen's inequality,
$\E\bigl[[G]^{2}_{\mathcal C^{1/2}(I)}\bigr]
\le\E\bigl[[G]^{4d}_{\mathcal C^{1/2}(I)}\bigr]^{1/(2d)}<\infty$, and hence
$\Lambda<\infty$.

The pair $(\tilde f,\tilde G)$ is a coupling of $f$ and $G$ as random elements of
$\mathcal C^{0}(I)$, so by Definition~\ref{def:w2}, \eqref{eq:step1} and \eqref{eq:Lambda},
\begin{equation}
  \label{eq:step2}
  \Wtwo^{\mathcal C^{0}(I)}(f,G)^{2}
  \;\le\;\E\bigl[\norm{h}^{2}_{\mathcal C^{0}(I)}\bigr]
  \;\le\;\frac{2\,(r+1)^{2d}D^{2}}{m}
       \;+\;2\,(r+1)^{2d}\,\Lambda\,\rho_k .
\end{equation}
Neither the left hand side nor the finite constant $\Lambda$ depends on $k$, and
$\rho_k=O(k^{-1})$ by \eqref{eq:covering}, so letting $k\to\infty$ in \eqref{eq:step2} removes the second term and leaves
\[
  \Wtwo^{\mathcal C^{0}(I)}(f,G)^{2}
  \;\le\;\frac{2\,(r+1)^{2d}D^{2}}{m}
  \;=\;\frac{(C^{\ast})^{2}}{m},
\]
since $D=\tau_a^{2}A^{2}\sqrt{M_{4p}/\nu}$ by \eqref{eq:Cgrowth}. Taking square roots gives
\eqref{eq:main}.
\end{proof}

\begin{proof}[Corollary~\ref{cor:compact}]

For $S\subseteq I$ and any $h\in\mathcal C^{0}(I)$ we have
$\norm{h}_{\mathcal C^{0}(S)}\le\norm{h}_{\mathcal C^{0}(I)}$. Applying this pathwise to
$h=f-G$ under any coupling and taking the infimum gives the first inequality;
Theorem~\ref{thm:functional} gives the second.

\end{proof}

\section{Beyond Polynomial Activations}\label{sec:general-activation}

For a general activation, an interpolation inequality of Gagliardo--Nirenberg type \citep{Nirenberg2011} bounds the supremum norm by the
$L^{2}(I)$ norm and the Hölder seminorm together.

\begin{proof}[Lemma~\ref{lem:interp}]
Write $\ell_j := \abs{I_j} > 0$, $\ell_{\min} := \min_j \ell_j$,
$L := [h]_{\mathcal{C}^\gamma(I)}$ and $M := \norm{h}_{\mathcal{C}^0(I)}$.
Since $h$ is continuous and $I$ is compact, $M$ is attained at some
$x_0 \in I$; replacing $h$ by $-h$ if necessary we may assume
$h(x_0) = M \geq 0$. If $M = 0$ the inequality is trivial. If $L = 0$,
then $h \equiv M$ , so
$\norm{h}_{L^2(I)} = M\sqrt{\Vol(I)} \geq M \ell_{\min}^{d/2}$ and already
the second term on the right-hand side is at least $2M$.
Assume from now on $M, L > 0$.

For $x \in I$, H\"older continuity gives
$h(x) \geq M - L \norm{x - x_0}^{\gamma}$. Choose
$\delta := \bigl(\tfrac{M}{2L}\bigr)^{1/\gamma}$, so that
$L\delta^\gamma = \tfrac{M}{2}$, and $\eta := \delta/\sqrt{d}$. Let
\begin{equation*}
  J := \prod_{j=1}^d J_j, \qquad
  J_j := [x_{0,j} - \eta,\, x_{0,j} + \eta] \cap I_j .
\end{equation*}
Every $x \in J$ satisfies $\norm{x - x_0} \leq \sqrt{d}\,\eta = \delta$,
hence $h(x) \geq \tfrac{M}{2}$ on $J$. Since $x_{0,j} \in I_j$, we get $\abs{J_j} \geq \min(\eta, \ell_j)
\geq \min(\eta, \ell_{\min})$, so $\Vol(J) \geq \min(\eta, \ell_{\min})^d$
and
\begin{equation*}
  \norm{h}_{L^2(I)}^2
  \geq \int_J h(x)^2\, dx
  \geq \frac{M^2}{4}\, \min(\eta, \ell_{\min})^d .
\end{equation*}

\emph{Case 1: $\eta \leq \ell_{\min}$.} Substituting $\eta$,
\begin{equation*}
  \norm{h}_{L^2(I)}^2
  \geq \frac{M^2}{4}\, d^{-d/2} \Bigl(\frac{M}{2L}\Bigr)^{d/\gamma}
  = \frac{M^{\frac{2\gamma+d}{\gamma}}}
         {2^{\frac{2\gamma+d}{\gamma}}\, d^{d/2}\, L^{d/\gamma}} ,
\end{equation*}
and raising both sides to the power $\tfrac{\gamma}{2\gamma+d}$ gives
\begin{equation*}
  M \leq 2\, d^{\frac{\gamma d}{2(2\gamma+d)}}\,
  L^{\frac{d}{2\gamma+d}}\,
  \norm{h}_{L^2(I)}^{\frac{2\gamma}{2\gamma+d}} .
\end{equation*}

\emph{Case 2: $\eta > \ell_{\min}$.} Then
$\norm{h}_{L^2(I)}^2 \geq \tfrac{M^2}{4}\, \ell_{\min}^d$, i.e.\
$M \leq 2\, \ell_{\min}^{-d/2}\, \norm{h}_{L^2(I)}$.

In either case $M$ is bounded by one of the two non-negative terms on the
right-hand side of the claim, hence by their sum.
\end{proof}

\end{document}